\documentclass[11pt]{article}
\usepackage[letterpaper, margin=2.54cm]{geometry}
\usepackage{paqueterias}
\usepackage[square,numbers]{natbib}
\usepackage{authblk}

\title{On the study of frequency control and spectral bias in Wavelet-Based Kolmogorov Arnold networks: A path to physics-informed KANs}

\author{J. Daniel Meshir \thanks{Electronic address: \texttt{juan.meshir2900@alumnos.udg.mx}; Corresponding author}}
\author{Abel Palafox \thanks{Electronic address: \texttt{abel.palafox@academicos.udg.mx}}}
\author{E. Alejandro Guerrero \thanks{Electronic address: \texttt{edgar.guerrero@academicos.udg.mx}}}

\affil{Department of Mathematics, University Center for Exact and Engineering Sciences\\University of Guadalajara}

\date{\today}

\begin{document}

\maketitle
\justifying
\begin{abstract}

Spectral bias, the tendency of neural networks to prioritize learning low-frequency components of functions during the initial training stages, poses a significant challenge when approximating solutions with high-frequency details. This issue is particularly pronounced in physics-informed neural networks (PINNs), widely used to solve differential equations that describe physical phenomena. In the literature, contributions such as Wavelet Kolmogorov Arnold Networks (Wav-KANs) have demonstrated promising results in capturing both low- and high-frequency components. Similarly, Fourier features (FF) are often employed to address this challenge. However, the theoretical foundations of Wav-KANs, particularly the relationship between the frequency of the mother wavelet and spectral bias, remain underexplored. A more in-depth understanding of how Wav-KANs manage high-frequency terms could offer valuable insights for addressing oscillatory phenomena encountered in parabolic, elliptic, and hyperbolic differential equations. In this work, we analyze the eigenvalues of the neural tangent kernel (NTK) of Wav-KANs to enhance their ability to converge on high-frequency components, effectively mitigating spectral bias. Our theoretical findings are validated through numerical experiments, where we also discuss the limitations of traditional approaches, such as standard PINNs and Fourier features, in addressing multi-frequency problems.

\end{abstract}

\section{Introduction}

In many fields such as industry, physics, mathematics, economics, and biology, understanding the relationship between features and outcomes of physical phenomena is of great interest \cite{differentialequationaplications1,Penrosetheroadtoreality,strogatz:2000}. In recent decades, Machine learning methods have gained prominence for detecting those relationships in complex systems \cite{KISSAS2020112623aplicationsDNN, PhysRevE025205applicationsDNN}. A notable machine learning that focuses on approximate solutions to differential equations that describe physical processes is \textit{Physics-Informed Neural Networks} (PINNs) \cite{RAISSI2019686,SIRIGNANO20181339}, embedding the physical laws governing the phenomena in the neural network structure to better approximate solutions. 

Function approximation capability of neural networks (NNs) is supported by the universal approximation theorem \cite{Cybenko1989ApproximationBS, HORNIK1990551approach}. This theorem asserts that a feedforward NN with sufficient neurons can approximate any continuous function within a compact domain, such as the hypercube $[0,1]^n$. While this provides a theoretical foundation for NNs as universal approximators, practical challenges arise, especially when approximating functions with high-frequency components \cite{FourierFeaturesTancik2020}. One major challenge is known as \textit{spectral bias} \cite{cao2020understandingspectralbiasdeep}, where NNs predominantly learn low-frequency components of a function in the initial training phases. In contrast, high-frequency components are learned much more slowly; in some cases, the training process may reach the maximum number of iterations before fully capturing these high-frequency details.

In some cases, spectral bias can be beneficial. For example, it helps NNs avoid over-fitting to noise in the data by focusing on the general patterns first \cite{badger2022deeplearninggeneralizes}. However, when dealing with more complex data, such as high-resolution images, signal processing tasks, or the need to approximate both low- and high-frequency components of a differential equation solution, this bias becomes problematic \cite{FourierFeaturesTancik2020, cao2020understandingspectralbiasdeep, Basrifrecuencybias2020, rahaman2019spectralbiasneuralnetworks}, NN requires many iterations to capture high-frequency details, leading to inefficient and computationally expensive training processes that are often impractical \cite{basri2019convergencerateneuralnetworks}. This issue is particularly relevant in PINNs, usually inherited from the traditional NNs upon which they are built \cite{krishnapriyan2021characterizingpossiblefailuremodes, ProblemsPINNs2, WANG2022110768whenandwhyfailspinns}. This bias hinders their ability to accurately represent high-frequency details, posing a significant challenge to their effectiveness. To mitigate spectral bias, researchers have proposed using \textit{Fourier Features} \cite{FourierFeaturesTancik2020, WANG2021113938fourierfeaturespINNs}. Fourier Features introduce a preprocessing layer that transforms the input data by encoding it with sinusoidal functions of various frequencies. This transformation embeds high-frequency components into the input representation, making it easier for the NN to learn both high- and low-frequency patterns. By adding these Fourier Feature layers, PINNs can better approximate complex solutions and overcome some limitations imposed by spectral bias. The success of Fourier Features highlights their potential to improve PINNs and other neural network architectures in tasks requiring detail. Despite their advantages, Fourier Features also come with specific challenges and limitations. First, the choice of frequency scales in the sinusoidal embedding is critical; poorly chosen scales can lead to over-parameterization, increased computational overhead, or difficulty in optimization. Furthermore, adding these features introduces extra hyperparameters, such as the frequency range and the number of frequencies, which require careful tuning to balance performance and computational efficiency.

Recently, Kolmogorov-Arnold Networks (KANs) \cite{liu2024kankolmogorovarnoldnetworks} have emerged as an innovative neural network architecture, grounded in the Kolmogorov-Arnold representation theorem \cite{KolmogorovArnold2009Theorem, KANrepresentationtheorem2}. This theorem asserts that any multivariate continuous function can be decomposed into a sum of univariate functions. Unlike traditional neural networks, which employ a fixed nonlinear activation function, KANs assign a learnable univariate function to each connection or "edge." These univariate functions can take various forms, like B-spline curves, leading to a variant known as Spl-KAN \cite{liu2024kankolmogorovarnoldnetworks}. Other potential choices include Chebyshev polynomials \cite{ss2024chebyshevpolynomialbasedkolmogorovarnoldnetworks} or, as explored in our study, wavelet functions, which give rise to Wavelet Kolmogorov-Arnold Networks (Wav-KANs) \cite{bozorgasl2024wavkanwaveletkolmogorovarnoldnetworks}. Experiments on the MNIST dataset demonstrate their efficacy, showcasing their potential for many applications involving intricate and hierarchical data structures.

In this paper, we focus on studying the capacity of Wav-KANs to approximate functions containing both low- and high-frequency components. We analyze this through the lens of the Neural tangent kernel (NTK) \cite{NTKjacot2018}, showing that the eigenvalues of the NTK can be controlled by adjusting the frequency of the mother wavelet used in the network. This frequency control enables us to modulate the convergence speed for high-frequency components during training, which has significant implications for mitigating spectral bias. To the best of our knowledge, this is the first theoretical and practical study of Wav-KANs through the NTK framework. Our findings show that carefully selecting the wavelet functions can optimize the learning process for functions with diverse frequency content, thus addressing a key challenge in NN training. This study contributes to advancing the theoretical understanding of wavelet-based neural networks. It provides a pathway for more efficient training techniques, leading to faster, more accurate solutions for problems in scientific computing, engineering, and beyond.

Moreover, Wav-KANs hold significant promise for replacing traditional NNs in PINNs, where the spectral bias inherent to standard architectures has been a significant limitation. The ability of Wav-KANs to effectively balance the learning of low- and high-frequency components suggests a path forward for overcoming these issues, offering enhanced performance and robustness in solving complex partial differential equations and other scientific problems.

The rest of the paper is organized as follows: Section 2 reviews the literature on spectral bias in neural networks and prior work on KANs, Wav-KANs, and PINNs. Section 3 presents the theoretical background of Wav-KANs, explains the spectral bias problem, and demonstrates how controlling the mother wavelet's frequency influences the NTK's eigenvalues, comparing this with other methods like Fourier features. Section 4 discusses the adaptation of Wav-KANs within the PINN framework and compares the results with different models. Section 5 explores broader implications and potential applications, and Section 6 provides conclusions.

\section{Related Work}

 One of the first comprehensive studies on spectral bias was conducted by Nasim Rahaman et al. in \cite{rahaman2019spectralbiasneuralnetworks}, where they provided empirical evidence that NNs exhibit a preference for learning low-frequency patterns first. This work laid the foundation for further investigations into the underlying mechanisms of spectral bias. In another important contribution, Ronen Basri et al. in \cite{basri2020frequencybiasneuralnetworks} explored spectral bias through the lens of the NTK, offering a theoretical explanation for this behavior. 

The term NTK was initially introduced by Jacot et al. in \cite{NTKjacot2018}, where they demonstrated that the evolution of a NN during training can be described by a kernel, now known as the NTK. This provides a robust framework for analyzing the training dynamics of neural networks in the infinite-width regime, where the training process can be approximated by a linear model. One key advantage of using the NTK to study spectral bias is that the eigenvalues are closely linked to the convergence speed of different frequency components during training. Larger eigenvalues correspond to faster learning of certain features, typically lower-frequency components, while smaller eigenvalues lead to slower convergence, which affects the learning of high-frequency components. This relationship makes it a compelling tool for understanding and potentially mitigating spectral bias in NN.

On the other hand, as mentioned earlier, with the recent development of KANs, Wav-KANs have demonstrated strong performance in capturing both high- and low-frequency components in image data, as shown in \cite{bozorgasl2024wavkanwaveletkolmogorovarnoldnetworks}. This work aims to extend this capability to function approximation for mitigating spectral bias. Additionally, Yizheng Wang et al. in \cite{wang2024kolmogorovarnoldinformedneural} explored the use of KANs within PINNs, creating Kolmogorov Arnold Informed Neural Networks (KINNs) by replacing NNs with KANs. Their approach, based on Spl-KANs \cite{liu2024kankolmogorovarnoldnetworks}, demonstrated promising results and reduced spectral bias in practical experiments. Although Wav-KANs have also shown effective performance within PINNs to approximate solutions to differential equations \cite{patra2024physicsinformedkolmogorovarnoldneural}, no studies, theoretical or practical, have yet examined spectral bias when using Wav-KANs for function approximation or Wav-KINNs for solving differential equations.

An alternative approach to mitigate spectral bias in NNs involves Fourier feature mapping. Matthew Tancik et al. \cite{FourierFeaturesTancik2020} proposed transforming input data using Fourier features to enable NNs to capture high-frequency image components. Building on this, Sifan Wang et al. \cite{WANG2021113938fourierfeaturespINNs} applied Fourier features (FF) to PINNs, showing theoretically that this technique allows NNs to capture high frequencies. They further proposed multiple Fourier features (MFF) to improve the performance of PINNs on specific differential equations. Recognizing that different functions may have distinct spatial and temporal frequency domains, they developed spatial-temporal multiple Fourier features (STMFF) by adding them separately for spatial and temporal domains. This approach effectively enabled PINNs to capture high and low frequencies in practical applications. However, while FF do not add more trainable parameters, they increase the model's hyper-parameters. In STMFF, hyper-parameters must be carefully chosen to approximate the target function accurately.

In this work, we examine the training dynamics of Wav-KANs through the lens of the NTK, an approach that, to our knowledge, has not been previously explored. Building on insights from \cite{WANG2021113938fourierfeaturespINNs}, we theoretically demonstrate that the frequencies learned by the network can be precisely controlled by the frequency of the mother wavelet, allowing us to address spectral bias. Additionally, we provide empirical evidence that increasing the number of hidden units can achieve a similar effect to adjusting the mother wavelet’s frequency. This simplifies the model by eliminating the need for additional hyper-parameters to capture both low- and high-frequency components, unlike approaches such as Fourier features.

Our findings contribute theoretically and practically by presenting a framework that seamlessly integrates into the PINN structure. We empirically show that Wav-KINNs effectively mitigate spectral bias, positioning them as a promising alternative to traditional PINNs for complex, multi-frequency function approximations. This is the first study to combine NTK analysis with wavelet-driven frequency control in Wav-KANs, bridging theoretical insights and practical implementation for spectral bias mitigation in scientific machine-learning applications.

\section{Wavelet Kolmogorov Arnold Networks and Spectral Bias}

The architecture of KANs is based on the Kolmogorov-Arnold Representation Theorem \cite{KolmogorovArnold2009Theorem}, which states that any continuous multivariate function can be expressed as a combination of continuous univariate functions. The theorem can be formally stated as follows:

\begin{theorem}[Kolmogorov-Arnold representation]
    For any continuous function $f: [0,1]^n\to \mathbb{R}$, there exist $2n+1$ functions $\Psi_q: \mathbb{R}\to \mathbb{R}$ and $2n+1\times n$ functions $\psi_{q,p}: [0,1]\to \mathbb{R}$, all univariate and
continuous, such that
\begin{equation}\label{eq:Kolmogorov-Arnoldtheorem}
    f(x_1,...,x_n)=\sum_{q=1}^{2n+1} \Psi_q\left(\sum_{p=1}^{n}\psi_{q,p}(x_p)\right).
\end{equation}
\end{theorem}

In recent work on Spl-KANs \cite{liu2024kankolmogorovarnoldnetworks}, a generalized version of the Kolmogorov-Arnold representation is introduced, extending KANs to deeper architectures. In KANs, each "weight" operates as a small learnable function rather than a fixed value, and each node does not use a fixed non-linear activation function. Instead, each learnable activation function on the edges processes inputs and produces outputs, with the univariate functions being trainable. Simple summations are performed at each node, making the Kolmogorov-Arnold representation theorem naturally align with a two-layer KAN architecture (illustrated in Figure \ref{fig:KANrepresentation}). Each learnable univariate function can be parameterized either as a B-spline curve, as initially proposed, or using a mother wavelet as the functional basis, as suggested in \cite{bozorgasl2024wavkanwaveletkolmogorovarnoldnetworks}.

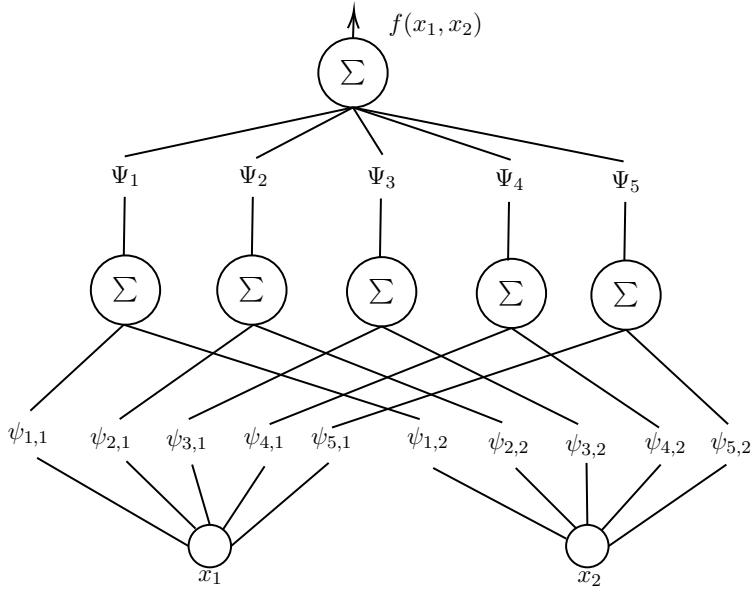
\begin{figure}[ht]
    \centering

\tikzset{every picture/.style={line width=0.75pt,scale=0.85,every node/.style={scale=0.85}}}

\begin{tikzpicture}[x=0.75pt,y=0.75pt,yscale=-1,xscale=1]

\draw   (202.8,337.6) .. controls (202.8,330.64) and (208.44,325) .. (215.4,325) .. controls (222.36,325) and (228,330.64) .. (228,337.6) .. controls (228,344.56) and (222.36,350.2) .. (215.4,350.2) .. controls (208.44,350.2) and (202.8,344.56) .. (202.8,337.6) -- cycle ;
\draw   (426.8,337.6) .. controls (426.8,330.64) and (432.44,325) .. (439.4,325) .. controls (446.36,325) and (452,330.64) .. (452,337.6) .. controls (452,344.56) and (446.36,350.2) .. (439.4,350.2) .. controls (432.44,350.2) and (426.8,344.56) .. (426.8,337.6) -- cycle ;
\draw    (112.92,285.2) -- (202.8,337.6) ;
\draw    (347.92,291.2) -- (426.92,333.2) ;
\draw    (206.92,327.2) -- (165.92,287.2) ;
\draw    (432.92,327.2) -- (396.92,291.2) ;
\draw    (204.92,289.2) -- (215.4,325) ;
\draw    (222.92,328.2) -- (247.92,290.2) ;
\draw    (228,337.6) -- (285.92,288.2) ;
\draw    (439.4,325) -- (438.92,288.2) ;
\draw    (447.92,328.2) -- (482.92,289.2) ;
\draw    (452,337.6) -- (521.92,289.2) ;
\draw   (144.8,185.56) .. controls (144.8,174.21) and (154,165) .. (165.36,165) .. controls (176.71,165) and (185.92,174.21) .. (185.92,185.56) .. controls (185.92,196.92) and (176.71,206.13) .. (165.36,206.13) .. controls (154,206.13) and (144.8,196.92) .. (144.8,185.56) -- cycle ;
\draw   (219.8,185.56) .. controls (219.8,174.21) and (229,165) .. (240.36,165) .. controls (251.71,165) and (260.92,174.21) .. (260.92,185.56) .. controls (260.92,196.92) and (251.71,206.13) .. (240.36,206.13) .. controls (229,206.13) and (219.8,196.92) .. (219.8,185.56) -- cycle ;
\draw   (296.8,186.56) .. controls (296.8,175.21) and (306,166) .. (317.36,166) .. controls (328.71,166) and (337.92,175.21) .. (337.92,186.56) .. controls (337.92,197.92) and (328.71,207.13) .. (317.36,207.13) .. controls (306,207.13) and (296.8,197.92) .. (296.8,186.56) -- cycle ;
\draw   (373.8,187.56) .. controls (373.8,176.21) and (383,167) .. (394.36,167) .. controls (405.71,167) and (414.92,176.21) .. (414.92,187.56) .. controls (414.92,198.92) and (405.71,208.13) .. (394.36,208.13) .. controls (383,208.13) and (373.8,198.92) .. (373.8,187.56) -- cycle ;
\draw   (441.8,188.56) .. controls (441.8,177.21) and (451,168) .. (462.36,168) .. controls (473.71,168) and (482.92,177.21) .. (482.92,188.56) .. controls (482.92,199.92) and (473.71,209.13) .. (462.36,209.13) .. controls (451,209.13) and (441.8,199.92) .. (441.8,188.56) -- cycle ;
\draw    (164.36,206.13) -- (108.92,257.2) ;
\draw    (339.92,262.2) -- (164.36,206.13) ;
\draw    (161.92,262.2) -- (241.36,206.13) ;
\draw    (388.92,264.2) -- (241.36,206.13) ;
\draw    (433.92,265.2) -- (317.36,207.13) ;
\draw    (202.92,262.2) -- (317.36,207.13) ;
\draw    (481.92,267.2) -- (394.36,208.13) ;
\draw    (250.92,265.2) -- (394.36,208.13) ;
\draw    (522.92,265.2) -- (462.36,209.13) ;
\draw    (287.92,267.2) -- (462.36,209.13) ;
\draw    (164.36,165) -- (164.92,130.2) ;
\draw    (240.36,165) -- (240.92,130.2) ;
\draw    (316.36,166) -- (316.92,131.2) ;
\draw    (392.36,168) -- (392.92,133.2) ;
\draw    (461.36,168) -- (461.92,133.2) ;
\draw   (279.8,56.56) .. controls (279.8,45.21) and (289,36) .. (300.36,36) .. controls (311.71,36) and (320.92,45.21) .. (320.92,56.56) .. controls (320.92,67.92) and (311.71,77.13) .. (300.36,77.13) .. controls (289,77.13) and (279.8,67.92) .. (279.8,56.56) -- cycle ;
\draw    (164.92,106.2) -- (300.36,77.13) ;
\draw    (242.92,107.2) -- (300.36,77.13) ;
\draw    (317.92,105.2) -- (300.36,77.13) ;
\draw    (393.92,108.2) -- (300.36,77.13) ;
\draw    (460.92,109.2) -- (300.36,77.13) ;
\draw    (300.36,36) -- (300.85,21.2) ;
\draw [shift={(300.92,19.2)}, rotate = 91.92] [color={rgb, 255:red, 0; green, 0; blue, 0 }  ][line width=0.75]    (10.93,-3.29) .. controls (6.95,-1.4) and (3.31,-0.3) .. (0,0) .. controls (3.31,0.3) and (6.95,1.4) .. (10.93,3.29)   ;

\draw (94,262.4) node [anchor=north west][inner sep=0.75pt]    {$\psi _{1,1}$};
\draw (143,265.4) node [anchor=north west][inner sep=0.75pt]    {$\psi _{2,1}$};
\draw (188,266.4) node [anchor=north west][inner sep=0.75pt]    {$\psi _{3,1}$};
\draw (234,265.4) node [anchor=north west][inner sep=0.75pt]    {$\psi _{4,1}$};
\draw (274,265.4) node [anchor=north west][inner sep=0.75pt]    {$\psi _{5,1}$};
\draw (331,265.4) node [anchor=north west][inner sep=0.75pt]    {$\psi _{1,2}$};
\draw (379,268.4) node [anchor=north west][inner sep=0.75pt]    {$\psi _{2,2}$};
\draw (425,269.4) node [anchor=north west][inner sep=0.75pt]    {$\psi _{3,2}$};
\draw (472,268.4) node [anchor=north west][inner sep=0.75pt]    {$\psi _{4,2}$};
\draw (511,268.4) node [anchor=north west][inner sep=0.75pt]    {$\psi _{5,2}$};
\draw (155,110.4) node [anchor=north west][inner sep=0.75pt]    {$\Psi _{1}$};
\draw (231,110.4) node [anchor=north west][inner sep=0.75pt]    {$\Psi _{2}$};
\draw (307,111.4) node [anchor=north west][inner sep=0.75pt]    {$\Psi _{3}$};
\draw (383,112.4) node [anchor=north west][inner sep=0.75pt]    {$\Psi _{4}$};
\draw (452,113.4) node [anchor=north west][inner sep=0.75pt]    {$\Psi _{5}$};
\draw (207,350.4) node [anchor=north west][inner sep=0.75pt]    {$x_{1}$};
\draw (432,351.4) node [anchor=north west][inner sep=0.75pt]    {$x_{2}$};
\draw (320,19.4) node [anchor=north west][inner sep=0.75pt]    {$f( x_{1} ,x_{2})$};
\draw (156,176.4) node [anchor=north west][inner sep=0.75pt]    {$\sum $};
\draw (231,176.4) node [anchor=north west][inner sep=0.75pt]    {$\sum $};
\draw (385,178.4) node [anchor=north west][inner sep=0.75pt]    {$\sum $};
\draw (453,179.4) node [anchor=north west][inner sep=0.75pt]    {$\sum $};
\draw (308,177.4) node [anchor=north west][inner sep=0.75pt]    {$\sum $};
\draw (291,47.4) node [anchor=north west][inner sep=0.75pt]    {$\sum $};

\end{tikzpicture}
    \caption{Kolmogorov-Arnold representation as KAN of [2,5,1] layers, where each is a learnable function, and each node is a simple summation.}
    \label{fig:KANrepresentation}
\end{figure}

\subsection{Wavelet Kolmogorov-Arnold networks (Wav-KANs)}

Based on the principles of wavelet transformation \cite{STEPHANE200989,STEPHANE20091}, the approach utilizes a base function known as a mother wavelet, which acts as a flexible template. This mother wavelet can be scaled and shifted, allowing each learnable function in the KAN to adapt dynamically to capture different patterns and frequencies in the data. Wav-KANs effectively build complex representations that capture low- and high-frequency components within a unified framework by tuning these scales and shifts. As a result, the forward pass in Wav-KANs is:

Let $\psi$ denote the mother wavelet, and let $x^{(l)}\inrn$, We then construct the matrices  $X^{(l)}$, $W^{(l+1)}$, $T^{(l+1)}$ and $S^{(l+1)}$ as follows:

\begin{equation*}
    X^{(l)}=\left(
    \begin{matrix}
        (x^{(l)})^T\\
        (x^{(l)})^T\\
        \vdots \\
        (x^{(l)})^T\\
    \end{matrix}
    \right)\in \mathbb{R}^{m\times n}, \hspace{0.2cm}
    W^{(l+1)}=\left(
    \begin{matrix}
        W^{(l+1)}_{11}& \cdots& W^{(l+1)}_{1n} \\
        \vdots &\ddots&\vdots \\
        W^{(l+1)}_{m1}& \cdots& W^{(l+1)}_{mn}\\
    \end{matrix}
    \right),
\end{equation*}

\begin{equation*}
    T^{(l+1)}=\left(
    \begin{matrix}
        T^{(l+1)}_{11}& \cdots& T^{(l+1)}_{1n} \\
        \vdots &\ddots&\vdots \\
        T^{(l+1)}_{m1}& \cdots& T^{(l+1)}_{mn}\\
    \end{matrix}
    \right), \hspace{0.2cm}
    S^{(l+1)}=\left(
    \begin{matrix}
        S^{(l+1)}_{11}& \cdots& S^{(l+1)}_{1n} \\
        \vdots &\ddots&\vdots \\
        S^{(l+1)}_{m1}& \cdots& S^{(l+1)}_{mn}\\
    \end{matrix}
    \right),
\end{equation*}
\noindent
where $(x^{(l)})^T \in \mathbb{R}^{1\times n}$ is the input of layer $l$, and $W^{(l+1)}$, $T^{(l+1)}$ and $S^{(l+1)}$ are the learnable parameters of the Wav-KAN for $l=0,...,L-1$. 

\noindent
Defining, $\psi^{(l+1)}_{ij}:=\psi\left (\displaystyle\frac{x^{(l)}_j-T^{(l+1)}_{ij}}{S^{(l+1)}_{ij}}\right)$, we can introduce the matrix $\Psi^{(l+1)}(X^{(l)})$ as:

\begin{equation*}
    \Psi^{(l+1)}(X^{(l)})=\left(
    \begin{matrix}
        W^{(l+1)}_{11}\psi^{(l+1)}_{11}& W^{(l+1)}_{12}\psi^{(l+1)}_{12}& \cdots& W^{(l+1)}_{1n}\psi^{(l+1)}_{1n} \\
        W^{(l+1)}_{12}\psi^{(l+1)}_{12}& W^{(l+1)}_{22}\psi^{(l+1)}_{22}& \cdots& W^{(l+1)}_{2n}\psi^{(l+1)}_{2n}\\
        \vdots & \vdots&\ddots&\vdots \\
        W^{(l+1)}_{m1}\psi^{(l+1)}_{m1}& W^{(l+1)}_{m2}\psi^{(l+1)}_{m2}& \cdots& W^{(l+1)}_{1n}\psi^{(l+1)}_{mn}\\
    \end{matrix}\right).
\end{equation*}

\noindent
Thus, $\Psi^{(l+1)}(X^{(l)})$ represents the learnable functions as scaled and shifted versions of the mother wavelet, connecting the $j-th$ neuron in layer $l$ to the $i-th$ neuron in layer $l+1$.

As defined in \cite{bozorgasl2024wavkanwaveletkolmogorovarnoldnetworks}, the operator $\tau_0$ acts on the matrix $\Psi^{(l+1)}(X^{(l)})$ by summing the elements in each row to produce the resulting vector $x^{(l+1)}$. This process is defined as follows:

\begin{equation*}
\begin{split}
    \tau_0 (\Psi^{(l+1)}(X^{(l)}))&= x^{(l+1)},\\
     x^{(l+1)}_i = \sum_{j=1}^n \Psi^{(l+1)}(X^{(l)})_{ij}&=\sum_{j=1}^n W^{(l+1)}_{ij}\psi^{(l+1)}_{ij}. 
\end{split}
\end{equation*}

If $X^{(0)}$ is taken as the input matrix containing only the input vector as rows, then, for the entire network, the output after $L$ layers is:

\begin{equation*}
\begin{split}
    f_{WK}(x)&=x^{(L)}= \tau_0 (\Psi^{(L)}(X^{(L-1)}))= \tau_0\left(\Psi^{(L)}\left(\begin{matrix}
        (\tau_0 (\Psi^{L-1}(X^{(L-2)})))^T\\
        (\tau_0 (\Psi^{L-1}(X^{(L-2)})))^T\\
        \vdots\\
        (\tau_0 (\Psi^{L-1}(X^{(L-2)})))^T\\
    \end{matrix}\right)\right)\\
    &=\cdots = \tau_0\left(\Psi^{(L)}\left(\begin{matrix}
        (\tau_0 (\Psi^{L-1}\cdots (\tau_0(\Psi^{(1)}(X^{(0)})))))^T\\
        (\tau_0 (\Psi^{L-1}\cdots (\tau_0(\Psi^{(1)}(X^{(0)})))))^T\\
        \vdots\\
        (\tau_0 (\Psi^{L-1}\cdots (\tau_0(\Psi^{(1)}(X^{(0)})))))^T\\
    \end{matrix}\right)\right).\\
\end{split} 
\end{equation*}

For example, we can approximate the function $f(x)=4x^5$ using a Wav-KAN with two layers  (see Figure \ref{fig:Wav-KANExample}) and three hidden units in the middle layer, resulting in a Wav-KAN structure with a [1, 3, 1] shape. Choosing the Morlet wavelet $\psi(x)=e^{-\frac{1}{2}x^2}\cos(2x)$ as the mother wavelet, the learnable parameters are $W^1,T^1,S^1$ and $W^2,T^2,S^2$. Using only a few parameters and training the model over $5000$ epochs with the Adam optimizer, the mean squared error (MSE) converges to approximately $0.0014$, demonstrating the model’s efficiency and accuracy in approximating the function across the target interval.  

\begin{figure}[ht]
    \centering
    \includegraphics[scale=0.9]{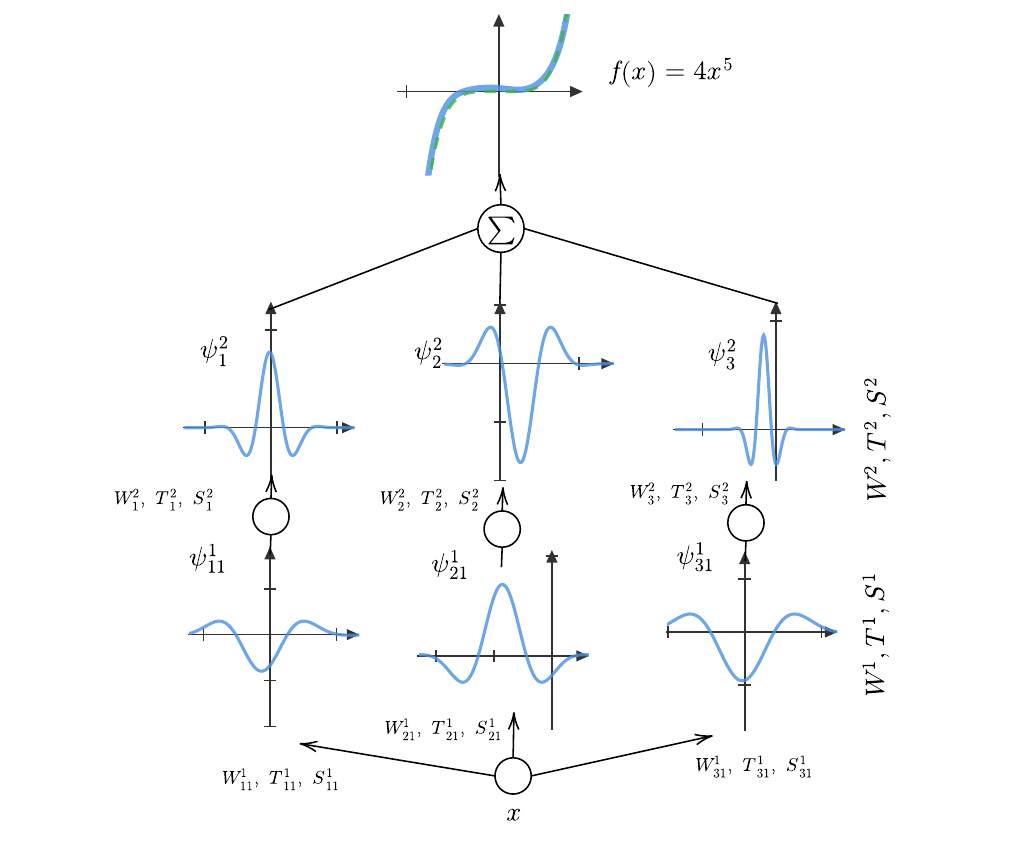}
    \caption{Example of a Wav-KAN used to approximate the function $f(x)=4x^5$ over the interval $[-1,1]$ with a [1, 3, 1] layer structure. Each edge represents a scaled and shifted instance of the mother wavelet, chosen as the Morlet wavelet function, defined by $e^{-\frac{1}{2}x^2}\cos(2x)$.}
    \label{fig:Wav-KANExample}
\end{figure}

\subsection{On the study of spectral bias through the lens of the Neural Tangent Kernel}
To analyze the capability of Wav-KANs in fitting functions with both low- and high-frequency components, an examination is conducted through the lens of the Neural Tangent Kernel (NTK). Building on the approach of \cite{WANG2021113938fourierfeaturespINNs}, we explore how Wav-KANs overcome spectral limitations by controlling the decay of NTK eigenvalues through the frequency of the mother wavelet. This tuning provides a direct mechanism for adjusting the NTK’s sensitivity to different frequency components, thus enabling a more balanced learning of both low- and high-frequency patterns. To illustrate this behavior, we first introduce the NTK matrix and analyze its properties in the context of Wav-KANs.

Following the methodology presented in \cite{WANG2021113938fourierfeaturespINNs} and \cite{wang2024kolmogorovarnoldinformedneural}, we adopt a similar structure for analyzing the behavior of our Wav-KAN. Let $f_{WK}(x;\theta(t))$ denote the Wav-KAN with parameters $\theta(t)=\{W^1(t),T^1(t),S^1(t),...,W^L(t),T^L(t),S^L(t)\}$, where $\theta(t)$ are the parameters of the network at iteration $t$. Given a training data set $\{x^r,y^r\}_{r=1}^N$, where $x^r$ are inputs, and $y^i$ are the corresponding outputs, we train the network by minimizing the least-squares loss function:

\begin{equation}\label{eq:squearelossfunction}
    \mathcal{L}(\theta)=\frac{1}{N}\sum_{r=1}^N |f_{WK}(x^r;\theta(t))-y^r|^2.
\end{equation}

Following Jacot et al. \cite{NTKjacot2018}, the neural tangent kernel operator $K$ by: 
\begin{equation}
    K(x^r,x^s)=\left<\frac{\partial f_{WK}(x^r;\theta(t))}{\partial \theta},\frac{\partial f_{WK}(x^s;\theta(t))}{\partial \theta}\right>.
\end{equation}

Notably, NTK theory shows that under gradient descent dynamics with an infinitesimally small learning rate (gradient flow), the kernel $K$ converges to a deterministic kernel $K_{ntk}$,  and remains constant during training as the width of the network grows to infinity.  This kernel $K_{ntk}$ is computed for all training points and forms an 
$N\times N$ positive semi-definite matrix. Furthermore, under the asymptotic conditions presented in \cite{liu2024kankolmogorovarnoldnetworks}, the dynamics $f_{WK}(X;\theta(t))=(f_{WK}(x^1;\theta(t)),...,f_{WK}(x^N;\theta(t)))^T$  relative to the target $Y=(y^1,...,y^N)^T$ follow:

\begin{equation}\label{eq:diffNTK}
    \begin{split}
        \frac{d f_{WK}(X;\theta(t))}{dt}=-K_{ntk}[f_{WK}(X;\theta(t))-Y].
    \end{split}
\end{equation}

Solving \eqref{eq:diffNTK}, we got

\begin{equation}\label{eq:solutionsystemNTK}
   f_{WK}(X;\theta(t))=(I-e^{-tK_{ntk}})Y.
\end{equation}

Since $K_{ntk}$ is a real symmetric positive semi-definite matrix, we can take its spectral decomposition $K_{ntk}=Q\Lambda Q^T$, where $Q$ is an
orthogonal matrix whose $i-th$ column is the eigenvector $q_i$ of $K_{ntk}$ and $\Lambda$ is a diagonal matrix whose diagonal entries $\lambda_i$ are the corresponding eigenvalues. Hence, we can rewrite \eqref{eq:solutionsystemNTK} as:

\begin{equation}
\begin{split}
f_{WK}(X;\theta(t))&= (I-Qe^{-\Lambda t}Q^T)Y,\\ 
Q^T[f_{WK}(X;\theta(t))-Y]&= -e^{-\Lambda t }Q^TY.
\end{split}
\end{equation}
Thus,

\begin{equation}
    \left(
    \begin{matrix}
        q_1^T\\
        q_2^T\\
        \vdots\\
        q_N^T
    \end{matrix}
    \right)\left(
    \begin{matrix}
        f_{WK}(x^1;\theta(t))-y^1\\
        f_{WK}(x^2;\theta(t))-y^2\\
        \vdots\\
        f_{WK}(x^N;\theta(t))-y^N
    \end{matrix}
    \right) =- \left(
    \begin{matrix}
        e^{-\lambda_1t} & & &\\
        & e^{-\lambda_2t} & &\\
        & & \ddots &\\
        & & & e^{-\lambda_Nt}
    \end{matrix}
    \right)Q^T Y.
\end{equation}

 We observe that the residual  $f(\textbf{x}^i;\theta)-\textbf{y}^i$
is inversely proportional to the eigenvalues of $K_{ntk}$. In standard neural networks, these eigenvalues typically decay as the frequency of the associated eigenfunctions rises, resulting in a slower convergence rate for the high-frequency components of the target function \cite{FourierFeaturesTancik2020,cao2020understandingspectralbiasdeep,rahaman2019spectralbiasneuralnetworks,WANG2021113938fourierfeaturespINNs}. This behavior, known as spectral bias, represents a fundamental limitation of deep networks, which tend to prioritize low-frequency features and struggle with higher frequencies.

Suppose we can demonstrate control over the decay rate of NTK eigenvalues in Wav-KANs by tuning the frequency of the mother wavelet. We can then directly influence the network's convergence speed across different frequency components. Specifically, this would enable Wav-KANs to prioritize and accelerate convergence on high-frequency components, thus solving the spectral bias problem. In the following, we analyze the behavior of Wav-KANs to evaluate this control mechanism.


To analyze the training behavior of Wav-KANs, consider a Wav-KAN model with a single layer, where $f_{WK}(x;\theta)$ is a function of input $x\inrn$ and target output $y\in R$ with $\theta=\{W,S,T\}$ representing the learning parameters of the model; then:

\begin{equation}
\begin{split}
    f_{WK}(x;\theta)&= \sum_{i=1}^n W_i \psi (\frac{x_i - T_i}{S_i})\\
    &=\sum_{i=1}^n W_i \psi_i(x_i)=\left<W,\Psi(X)\right>,
\end{split}
\end{equation}
where, $W$, $T$ and $S$ $\in \mathbb{R}^{1\times n}$. Here, $W$ represents the weight parameters, $T$ the translation parameters, and $S$ the scaling parameters for each input dimension, and each $\psi_i(x)$ is a scaled and shifted version of the mother wavelet $\psi$.

The kernel matrix $K$ is then defined by:
\begin{equation}
    K_{rs} =  K(x^r,x^s)= \left<\frac{\partial f_{WK}(x^r,\theta)}{\partial \theta}, \frac{\partial f_{WK} (x^s,\theta)}{\partial \theta} \right>,
\end{equation}
with $x^r$ and $x^s$ inputs.

Calculating the full kernel matrix $K$ can be complex when 
$S$ and $T$ are variable, so as a simplification, we consider a sub-case where $\theta=\{W\}$ and keep $T$ and $S$ fixed. Under this assumption, the kernel operator becomes:

\begin{equation}\label{eq:kerneldotproduct}
\begin{split}
    K(x^r,x^s)&= \left<\frac{\partial f_{WK}(x^r,\theta)}{\partial \theta}, \frac{\partial f_{WK} (x^s,\theta)}{\partial \theta} \right>,\\
    &=\left<\Psi(X^r),\Psi(X^s)\right>,\\
    &= \sum_{i=1}^n \psi_i(x_i^r)\psi_i(x_i^s).
\end{split}
\end{equation}

To further study the kernel’s eigensystem, we analyze the behavior of $K$ in the limit as the number of data points approaches infinity. In this limit, the eigensystem of the finite kernel matrix $K$ approximates the eigensystem of the function-based kernel $K(x^r,x^s)$, which satisfies the following integral equation \cite{equationKerneloperator2005ShaweTylor,WANG2021113938fourierfeaturespINNs}:

\begin{equation}\label{eq:integralequation}
    \int_C K(x^r,x^s)g(x^s)dx^s=\lambda g(x^r),
\end{equation}
where $C$ is a compact set, and $\lambda$ and $g(x^r)$ are the eigenvalue and eigenfunction of the kernel, respectively. This equation allows us to study how the spectral properties of the kernel, including its eigenvalues, depend on the choice of the mother wavelet and its scaling parameters, thus providing insights into the convergence behavior of Wav-KANs.

Deriving an equation for the eigenfunction $g(x^r)$ corresponding to non-zero eigenvalues is often a complex task, and the resulting expression is generally more intricate than what is typically seen for Fourier feature-based eigenfunctions \cite{WANG2021113938fourierfeaturespINNs}. Obtaining such an expression requires making several reasonable assumptions, which we outline below.

\begin{proposition}
\label{pro1}
Let $K(x^r,x^s)=\sum_{i=1}^n \psi_i(x_i^r)\psi_i(x_i^s)$, then the eigenfunction $g(x^r)$ corresponding to non-zero eigenvalues satisfying the the following equation
\begin{equation}\label{eq:g(x)equation}
    \left<\omega_0(x^r),\nabla g(x^r)\right>=g(x^r),
\end{equation}
where $\omega_0(x^r)=(h_1(x_1^r),h_2(x_2^r),...,h_n(x_n^r))^T$ and 
\begin{equation*}
    h_i(x_i^r)=\left\{ \begin{matrix}
        \frac{S_i}{\omega_i(x_i^r)} & & \omega_i(x_i^r)\neq 0\\
        0 & & otherwise
    \end{matrix}\right. ,
\end{equation*}
where $\omega_i(x_i^r)$ satisfies the equation:
\begin{equation*}
    \psi_i'(x_i^r)=\omega_i(x_i^r)\psi_i(x_i^r),
\end{equation*}
for $i=1,...,n$.
\end{proposition}

\begin{proof}
    The proof can be found in Appendix \ref{Appendix A}.
\end{proof}

Solving the equation \eqref{eq:g(x)equation} presented in Proposition \ref{pro1} can become intractable due to its complexity. Therefore, as done in \cite{WANG2021113938fourierfeaturespINNs}, we will consider a simplified scenario to gain insights into the behavior of the eigenfunctions and their corresponding eigenvalues. Specifically, we will focus on the case where $x\inr$, and the compact domain $C=[0,1]$ lead to Propositions \ref{prop2} and \ref{prop3} as described below.

\begin{proposition}\label{prop2}
    Let $K(x^r,x^s)= \psi_1(x^r)\psi_1(x^s)$, then the corresponding eigenfunctions $g(x^r)$ must have the form of:
    \begin{equation}\label{eq:g(x)}
        g(x^r)=C_1 \psi_1^{\frac{1}{S}}(x^r),
    \end{equation}
where $C_1$ is a constant and $S$ is the scale parameter to the mother wavelet.  

\end{proposition}

\begin{proof}
    The proof can be found in Appendix \ref{Appendix B}.
\end{proof}

For this analysis, we consider the compact domain $C=[0,1]$ and select the Morlet wavelet as the mother wave function. Substituting the expression from equation \eqref{eq:g(x)} in Proposition \ref{prop2} into the integral equation \eqref{eq:integralequation} generally leads to a highly complex formulation. Although the integrand can be explicitly defined, it may lack an elementary antiderivative, requiring numerical integration to determine the precise value of the eigenvalue $\lambda$. However, we can approximate the eigenvalue by making assumptions about the function’s domain and constraining the scaling and translation parameters, thereby enabling us to construct the target function to be learned by the Wav-KAN.

\begin{proposition}\label{prop3}
Let $K(x^r,x^s)= \psi_1(x)\psi_1(x^s)$ and $\psi(x)=e^{-\frac{1}{2}x^2}cos(bx)$, $S=1$ and $T\in [0,1]$, then the non-zero eigenvalue $\lambda$ is lower bounded by:
\begin{equation}
    \frac{1}{4}e^{-4(b\frac{x^r-T}{S})^2}\leq \lambda,
\end{equation}
and if $x^r\in [0,1]$ then:
\begin{equation}\label{eq:controllingfreq}
    \frac{1}{4}e^{-4(b({1-T}))^2}\leq \lambda.
\end{equation}
\end{proposition}

\begin{proof}
    The proof can be found in Appendix \ref{Appendix C}.
\end{proof}

As demonstrated in Proposition \ref{prop3}, the decay rate of the eigenvalues can be modulated by adjusting the frequency of the selected mother wave function, in this case, the Morlet wavelet. This control over the decay rate provides a means to counter spectral bias.

Although this simplified example provides insights, we anticipate that similar behavior will extend to the general architecture of Wav-KANs. However, computing general eigenvalues and eigenfunctions for a broader Wav-KAN configuration is even more complex than for Fourier features. Consequently, empirical verification of these results becomes essential, and we do so by examining specific cases of target functions that exhibit both high- and low-frequency components.

\subsection{Controlling NTK eigenvalues via mother wavelet frequency}

We analyze the NTK's eigenvalues for a function with both high- and low-frequency components, which are often challenging for deep neural networks to approximate effectively \cite{WANG2021113938fourierfeaturespINNs}. The target function is given by:

\begin{equation}\label{eq:solpoissoneq}
    u(x)=\sin(2\pi x)+0.1\sin(50\pi x),
\end{equation}
with $x\in [0,1]$. 

To explore this, we initialize a Wav-KAN with two layers and architecture $[1,35,1]$, using the Morlet wavelet $\phi(x) = e^{-\frac{1}{2}x^2} \cos(bx)$ as the mother wavelet. To approximate the target function \eqref{eq:solpoissoneq}, we use 100 data points $x \in [0,1]$ that are spaced, with the corresponding target values $y = u(x)$. The trainable parameters are initialized as $\theta = {W^1, T^1, S^1, W^2, T^2, S^2}$.

We perform NTK eigendecomposition on the Wav-KAN while varying the wavelet frequency parameter $ b $. Specifically, we evaluate frequencies $ b = 1, 5, 10, 15 $, and $ 25 $, as shown in Figures \ref{fig:controllingbyb1} and \ref{fig:controllingb1525}. After training the model for 1000 epochs, we observe in Figure \ref{fig:sinwavefreqb} that when the frequency $ b $ is low (e.g., $ b = 1 $ and $ b = 5 $, represented by the orange and green lines, respectively), the Wav-KAN struggles to approximate the high-frequency components of the target function. However, as $ b $ increases to $ b = 10 $ and $ b = 15 $, represented by the black and red lines, the model demonstrates an improved ability to capture finer details of the target function. This improvement is particularly evident in the zoomed-in frame, where the black line ($ b = 10 $) begins to deviate from the general form to capture additional details. However, the red line ($ b = 15 $) shows the most significant improvement, as the model successfully reproduces the high-frequency details.

In Figure \ref{fig:sinwaveeigenvaluesfreqb}, we observe that the NTK eigenvalues for higher values of \( b \) decay more slowly compared to those for \( b = 1 \) or \( b = 5 \). Notably, it is only at \( b = 15 \) that the eigenvalue decay reaches a point where the model can effectively reproduce the target solution. Additionally, as shown in Figure \ref{fig:sinwavelossfreqb}, the loss function for \( b = 15 \) decreases significantly faster than at lower frequency values, leading to a more accurate approximation of the target function \eqref{eq:solpoissoneq}.

\begin{figure}[ht]
    \centering
    \subfigure[]{
    \includegraphics[scale=0.45]{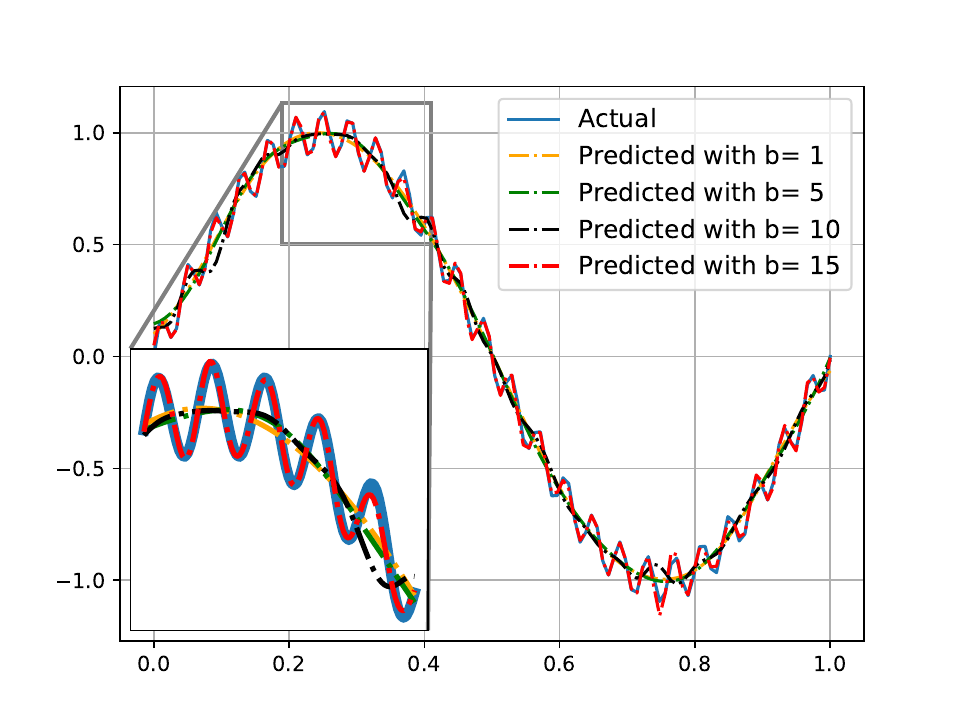}
    \label{fig:sinwavefreqb}
    }
    \subfigure[]{
    \includegraphics[scale=0.45]{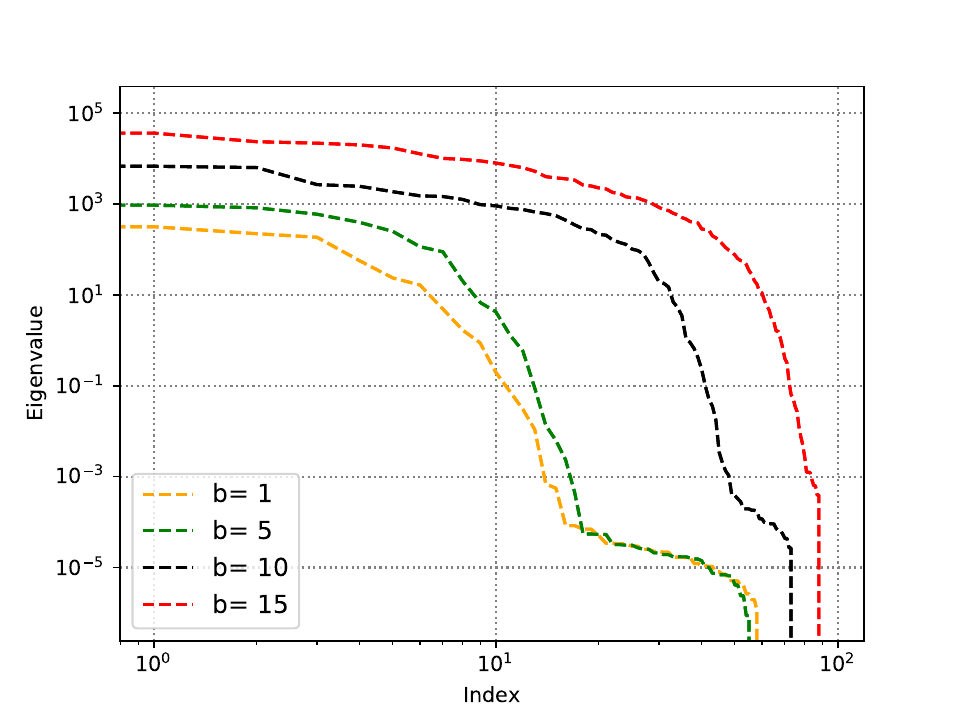}
    \label{fig:sinwaveeigenvaluesfreqb}
    }
        \subfigure[]{
    \includegraphics[scale=0.45]{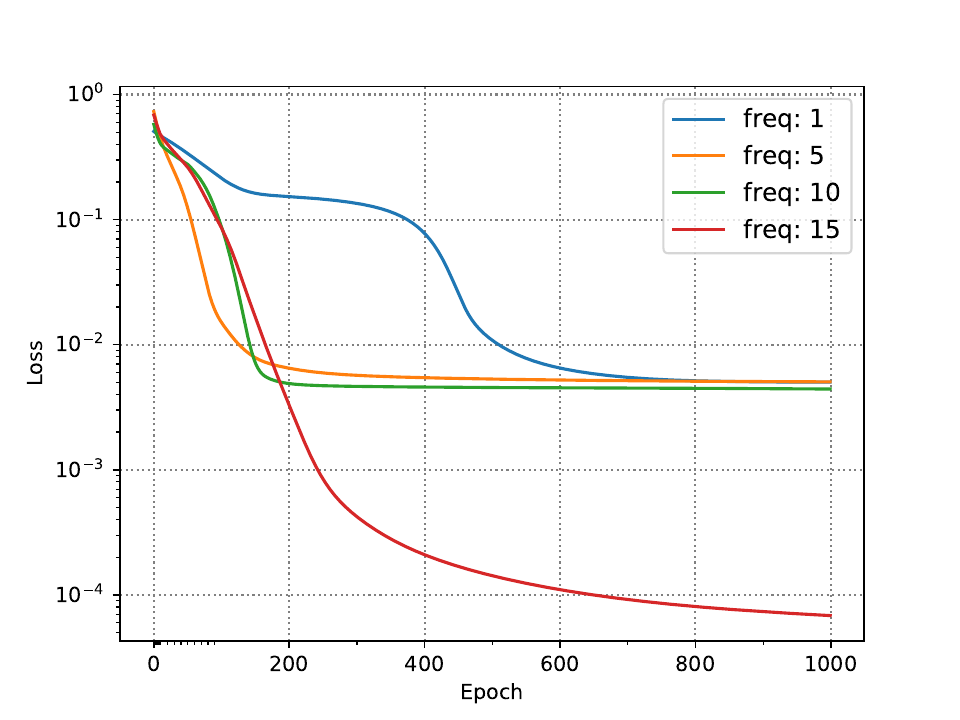}
    \label{fig:sinwavelossfreqb}
    }
    
\caption{Analyzing Wav-KAN behavior through the NTK. a) Approximation of function \eqref{eq:solpoissoneq} with high- and low-frequency components using a two-layer Wav-KAN $[1,35,1]$ and Morlet wavelet $\psi(x) = e^{-x^2}\cos(bx)$ for different $b$ values: $1, 5, 10,$ and $15$. b) NTK eigenvalues in descending order for these values of $b$. c) Loss evolution for varying $b$ values.}
    \label{fig:controllingbyb1}
\end{figure}

On the other hand, if we continue to increase the value of $b$, we observe a notable effect on the NTK's eigenvalues. As shown in Figure \ref{fig:eigenvalues_1525}, the eigenvalues decay more slowly with increasing $b$, which allows the model to learn high-frequency components more quickly. However, this can introduce an undesirable effect on the model's fit, as seen in Figure \ref{fig:controllingb1525}. When $b = 25$, the model displays signs of overfitting to the training dataset (see Figures \ref{fig:sinwaveapproach25} and \ref{fig:losssinwaveapproach25}), suggesting that while high-frequency components are indeed learned faster, excessively high values of $b$ may lead to overfitting.

\begin{figure}[ht]
    \centering
    \includegraphics[scale=0.45]{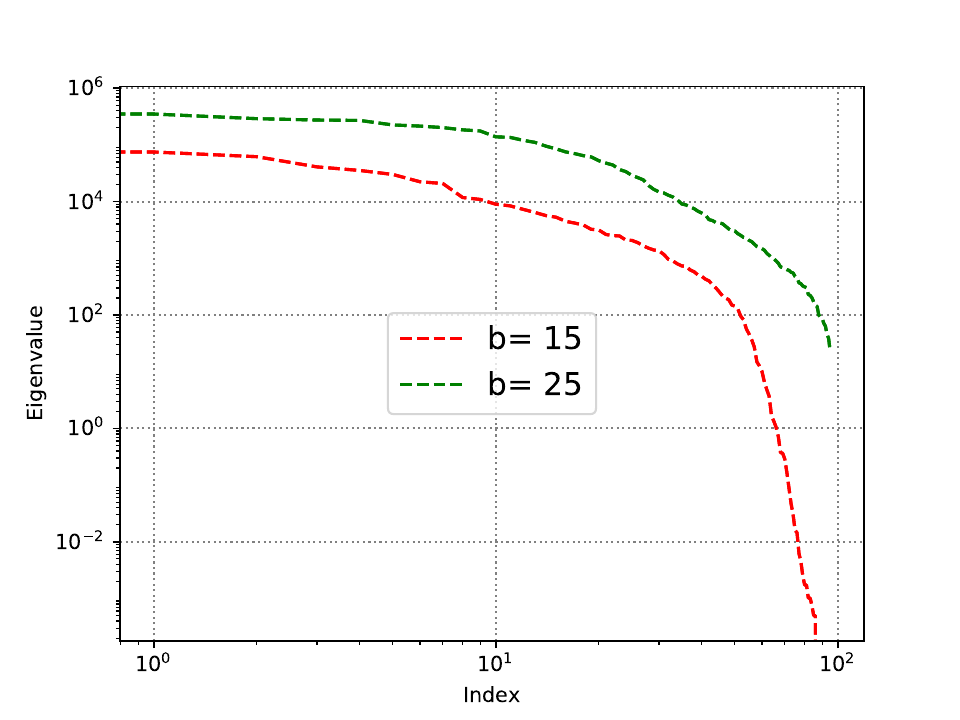}
    \caption{NTK eigenvalues in descending order for $b=15$ and $b=25$.}
    \label{fig:eigenvalues_1525}
\end{figure}

These results show that adjusting the frequency of the mother wavelet offers a means to control the decay rate of the NTK's eigenvalues. By tuning this parameter, we can influence how rapidly high-frequency components are learned, effectively introducing a new hyperparameter. The frequency parameter $b$ can be empirically selected based on learning curves to optimize performance and prevent overfitting. Thus, $b$ is a valuable hyperparameter in balancing the trade-off between learning speed for high frequencies and generalization to avoid overfitting.

\begin{figure}[ht]
    \centering
    \subfigure[]{
    \includegraphics[scale=0.44]{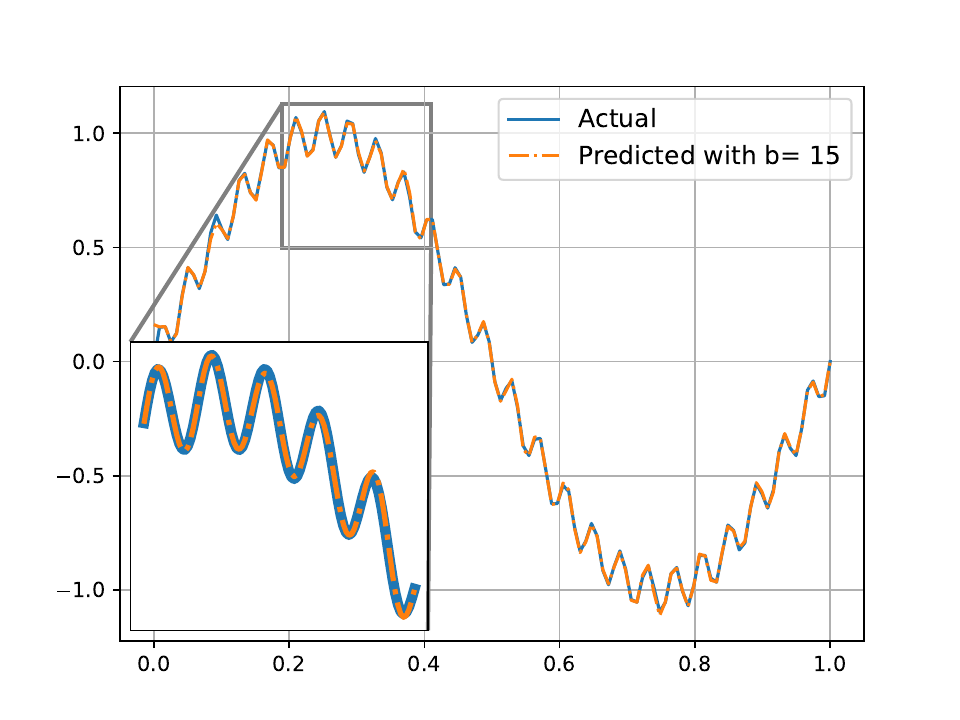}
    \label{fig:sinwaveapproach15}
    }
    \subfigure[]{
    \includegraphics[scale=0.44]{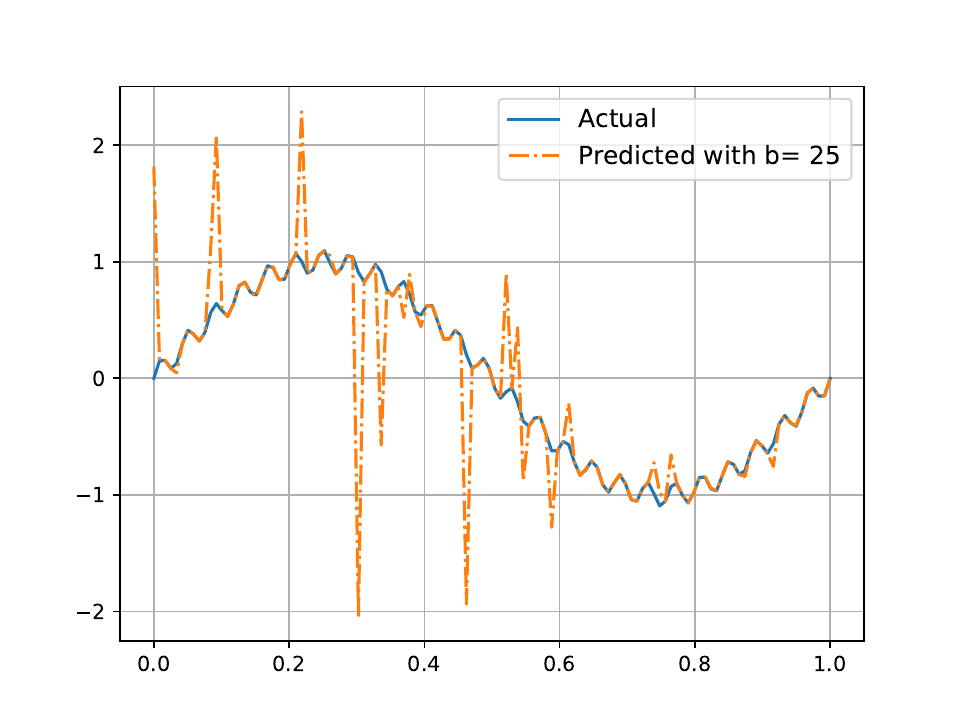}
    \label{fig:sinwaveapproach25}
    }
        \subfigure[]{
    \includegraphics[scale=0.44]{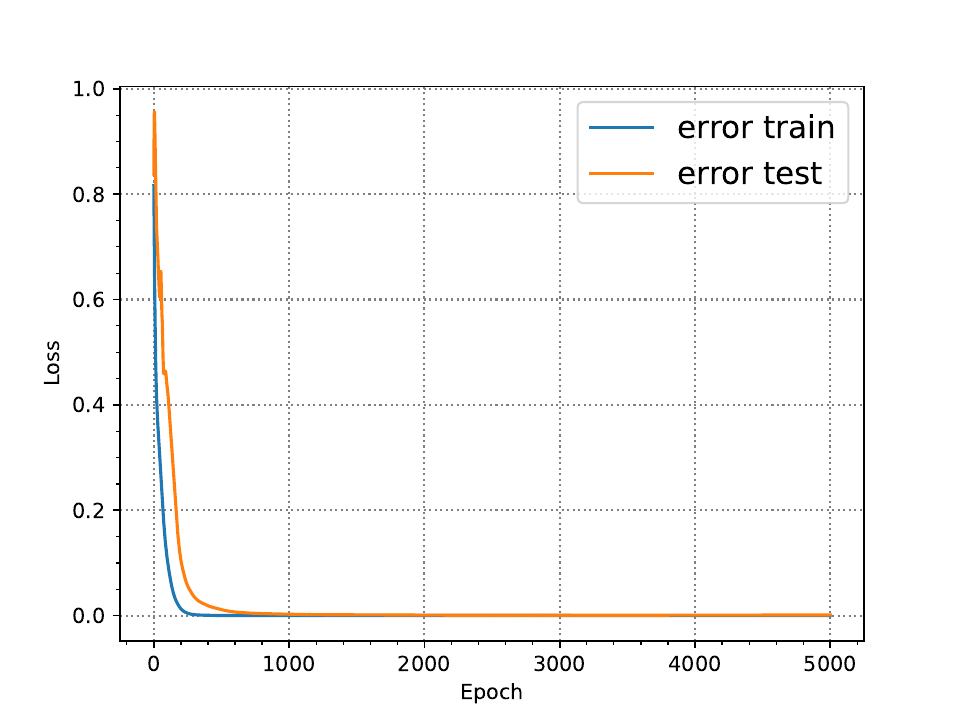}
    \label{fig:loossinwaveapproach15}
    }
        \subfigure[]{
    \includegraphics[scale=0.44]{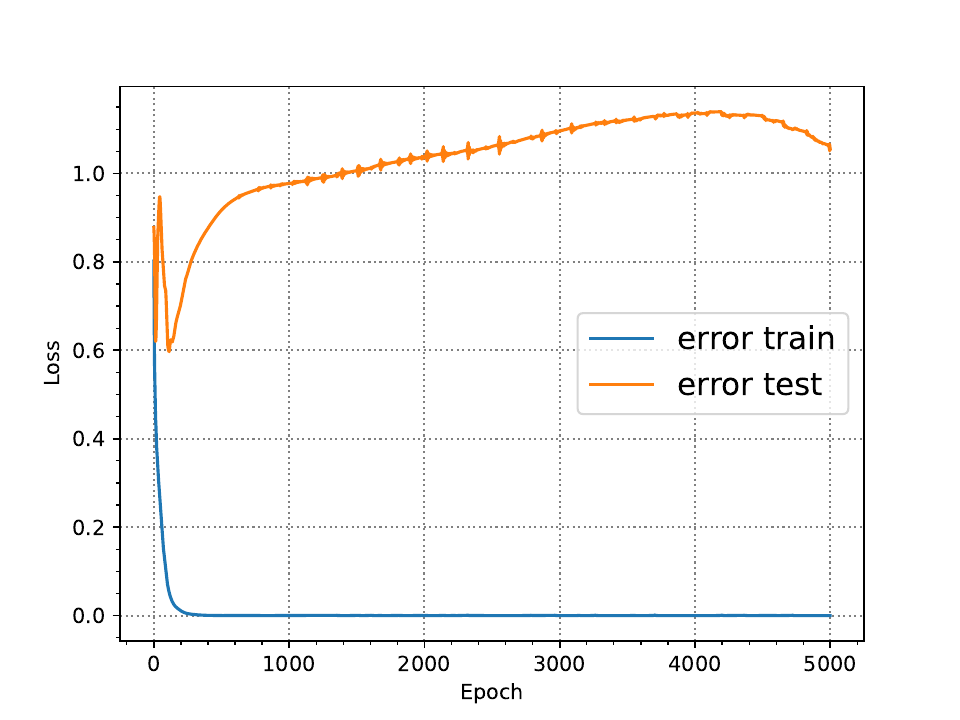}
    \label{fig:losssinwaveapproach25}
    }
    
\caption{a) Function approximation for \eqref{eq:solpoissoneq} with $b=15$. b) Function approximation for \eqref{eq:solpoissoneq} with $b=25$. c) Loss evolution for $b=15$. d) Loss evolution for $b=25$.}
    \label{fig:controllingb1525}
\end{figure}

While Wav-KANs theoretically require more trainable parameters than traditional NNs, as they involve weights, translations, and scalars ($W, T,$ and $S$), in practice, they often need fewer parameters to achieve comparable or superior performance \cite{bozorgasl2024wavkanwaveletkolmogorovarnoldnetworks, wang2024kolmogorovarnoldinformedneural}. For instance, to approximate the function in equation \eqref{eq:solpoissoneq}, we use a Wav-KAN with two layers and 35 hidden units, trained over 1000 epochs using the Adam optimizer with a learning rate of $0.001$ and $b=15$. In contrast, an NN requires four layers, each with 100 hidden neurons, and over 40,000 training epochs to reach a similar level of accuracy \cite{WANG2021113938fourierfeaturespINNs}.

Furthermore, in Proposition \ref{prop3}, assuming a domain of $x \in [0,1]$, we set $S = 1$ and $T \in [0,1]$, which is reasonable given the function's domain. By fixing $S_i^2 = S_j^2 = 1$ for all $i,j$ and setting $T^1_i$ and $T^2_i$ to random values within $[0,1]$, the only trainable parameters are $W^1$ and $W^2$. This reduces the complexity of Wav-KANs to be comparable with that of a standard NN.

In this setup, with fixed parameters $T$ and $S$, and using the Morlet wavelet $\psi(x) = e^{-x^2} \cos(bx)$ with $b=15$, we found that a Wav-KAN with two layers and a configuration of $[1,64,1]$ achieves similar accuracy to that of the Wav-KAN with non-fixed parameters shown in Figure \ref{fig:sinwaveapproach25}. This outcome suggests that fixing translations and scales in Wav-KANs can yield high accuracy while simplifying the model's training requirements. 

On the other hand, in Appendix \ref{Appendix D}, we briefly examine the behavior of the NTK when using different wavelet functions as the mother wavelet. By adjusting the frequency parameters of the mother wavelet, we can effectively control the decay rate of the NTK's eigenvalues. This exploration includes wavelets both with and without explicit frequency parameters, demonstrating that the choice of wavelet significantly impacts the spectrum of the NTK. Our results indicate that manipulating the wavelet's frequency characteristics offers a flexible way to influence the convergence properties of the neural network, thereby enhancing its capacity to capture various features in the data. 

\subsection{Controlling NTK eigenvalues via hidden units}
As previously discussed, the eigenvalues of the NTK can be adjusted by tuning the frequency of the mother wavelet, introducing a new hyperparameter akin to the $\sigma$ parameter in Fourier Features (FF) methods \cite{FourierFeaturesTancik2020, WANG2021113938fourierfeaturespINNs}. However, choosing an inappropriate frequency can lead to issues like overfitting, ultimately affecting the accuracy and generalization of the model. To address this, it’s helpful to analyze the behavior of Wav-KANs with a fixed mother wavelet frequency to avoid the need for extensive hyperparameter tuning.

In this experiment, we configure a Wav-KAN with two layers in the shape $[1, n, 1]$, where we vary the hidden units, $n$, to study its effects. The mother wavelet chosen is the Morlet function, $\psi(x) = e^{-x^2}\cos(5x)$, with a fixed frequency of $b = 5$. Figure \ref{fig:hiddenunits} illustrates that the Wav-KAN better approximates high-frequency components in function \eqref{eq:solpoissoneq} as the number of hidden units increases. This effect is comparable to increasing the frequency $b$ in previous experiments but without the need for modifying frequency directly. Instead, increasing hidden units provides finer control over the function’s high-frequency details.

Moreover, as shown in Figure \ref{fig:eigenvalueshiddenunints}, the NTK eigenvalues decay more gradually with higher numbers of hidden units, suggesting that the Wav-KAN effectively learns high-frequency components without inducing overfitting or other adverse effects. This indicates that by simply adjusting the network’s architecture (i.e., the number of hidden units), Wav-KANs can reduce spectral bias and approximate complex functions. This architecture-based approach to controlling spectral bias highlights the flexibility of Wav-KANs, making them promising models for capturing intricate frequency patterns in data while maintaining robust generalization.  

\begin{figure}
    \centering
    \subfigure[]{
    \includegraphics[scale=0.44]{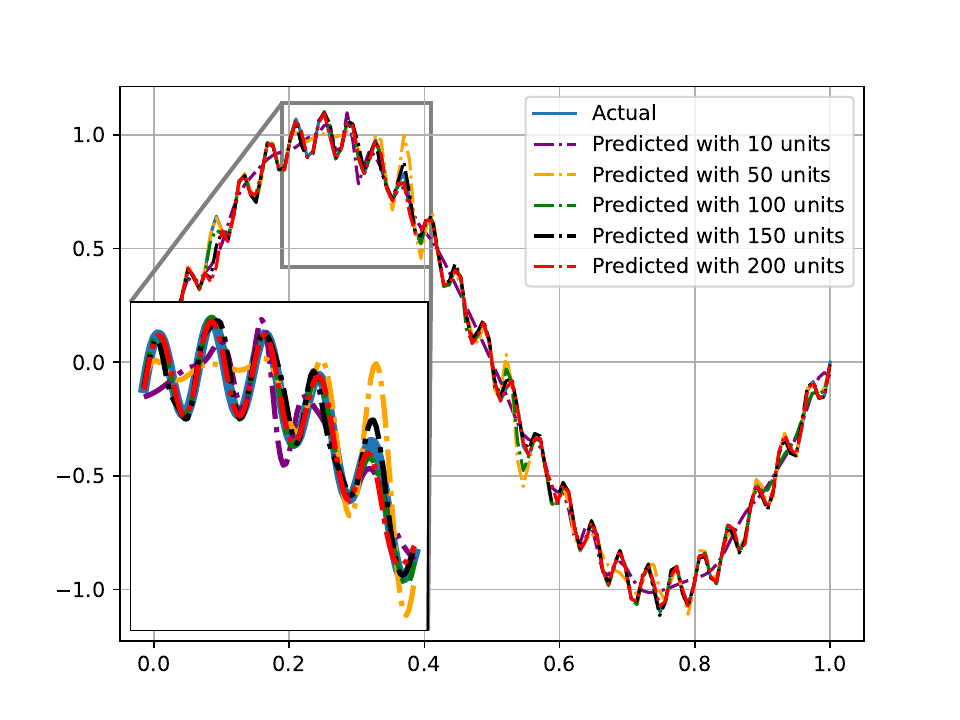}
    \label{fig:sinwavehiddenunits}
    }
    \subfigure[]{
    \includegraphics[scale=0.44]{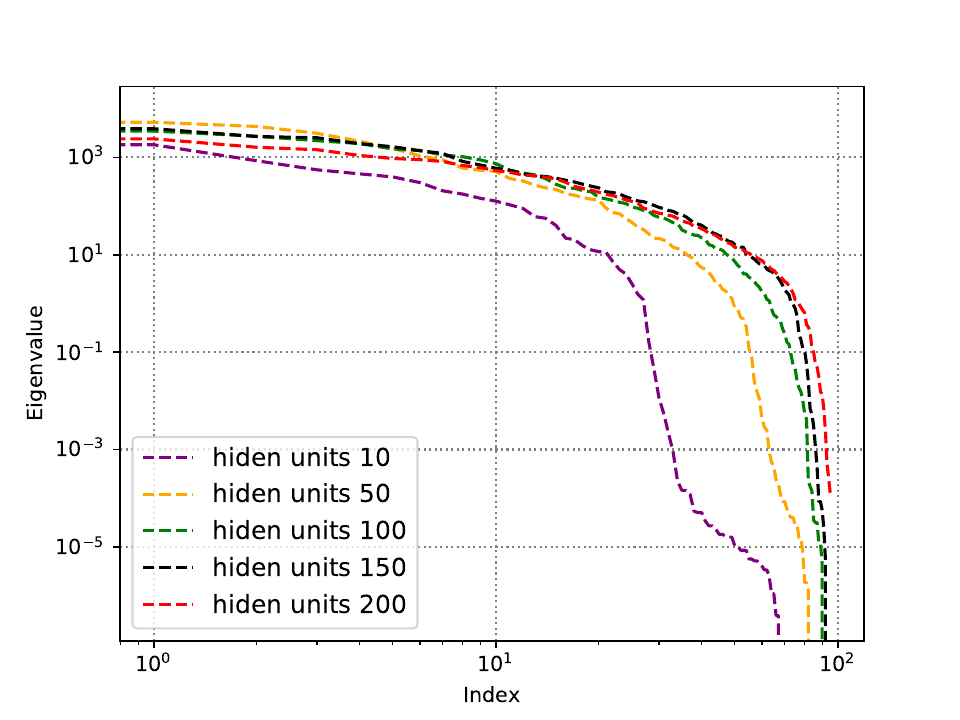}
    \label{fig:eigenvalueshiddenunints}
    }
\caption{Approximation of function \ref{eq:solpoissoneq} using a Wav-KAN with a fixed frequency Morlet mother wavelet. a) Approximation by a two-layer Wav-KAN with structure $[1,n,1]$, where $n=10, 50, 100, 150, 200$, trained for 1000 epochs. b) NTK eigenvalues in descending order for each value of $n$.}
    \label{fig:hiddenunits}
\end{figure}

\section{Extending to Physics-Informed Neural Networks}

Given that Wav-KANs naturally mitigate spectral bias and provide a mechanism to control NTK eigenvalue decay rates, they present a promising model for addressing spectral bias in PINNs, which typically suffer from this issue due to deep neural networks. Thus, replacing deep NNs with Wav-KANs in PINNs can help address spectral bias, leading to the development of models like Wav-KINNs \cite{wang2024kolmogorovarnoldinformedneural} or PIKANs \cite{patra2024physicsinformedkolmogorovarnoldneural}. This approach enables the approximation of solutions to differential equations containing both high- and low-frequency components without adding hyperparameters or partitioning the equation’s domain, as STMFF requires.

We begin with a brief overview of the implementation of Wav-KINNs for approximating solutions to differential equations.

Consider differential equations of the general form:

\begin{equation}
\begin{split}
    \mathcal{N}[u(x)]&=g(x), \text{ }x\in D,\\
    \mathcal{B}[u(x)]&=h(x), \text{ }x\in \partial D,
\end{split}
\end{equation}

\noindent
where $\mathcal{N}$ and $\mathcal{B}$ are the domain operator and boundary operator of the differential equations, respectively (potentially nonlinear), and $D$ and $\partial D$ represent the domain and boundary, respectively.  The solution of the differential equation is denoted by $u(x)$. Let $f(x;\theta)$ be a Wav-KAN with parameters $\theta$. We aim to approximate $u(x)$ using $f(x;\theta)$. For this purpose, we define a loss function $\mathcal{L}_\theta$ as follows:

\begin{equation}
\begin{split}
    \mathcal{L}_{D}&=\frac{1}{N_D} \sum_{i=1}^{N_D} |\mathcal{N}[f(x;\theta)]-g(x)|^2,\\
    \mathcal{L}_{\partial D}&=\frac{1}{N_\partial D} \sum_{i=1}^{N_\partial D} |\mathcal{B}[f(x;\theta)]-h(x)|^2,\\
    \mathcal{L}_\theta&=\mathcal{L}_D+\mathcal{L}_{\partial D}.
\end{split}
\end{equation}
We then optimize this loss function to obtain the parameters of the Wav-KAN that best approximate $u(x)$. This approach will be referred to as Wav-KINNs.

Wav-KINNs have demonstrated exemplary performance in approximating solutions for ordinary, system, and partial differential equations \cite{patra2024physicsinformedkolmogorovarnoldneural}.  However, to our knowledge, unlike Spl-KINNs, their ability to mitigate spectral bias has not yet been explored. This work investigates whether Wav-KINNs can address spectral bias by testing them on the three primary types of partial differential equations (elliptic, parabolic, and hyperbolic), all with high-frequency components, where PINNs typically struggle. We start this study with the 1D Poisson equation, whose solution, as expressed by the function \eqref{eq:solpoissoneq}, has been the focus of our analysis throughout this work.

\begin{equation}\label{eq:poissonequation}
\begin{split}
u_{xx}=-4\pi^2\sin(2\pi x)-& 0.1(50\pi)^2\sin(50\pi x)\text{ }x\in [0,1],\\
u(0)&=u(1)=0.\\
\end{split}
\end{equation}

We implemented a Wav-KINN with two layers in the architecture $[1,64,1]$, using the Morlet mother wavelet defined by $\psi(x)=e^{-x^2}cos(bx)$ with $b=10$. The model was trained with the Adam optimizer, using a learning rate of $0.001$  for 10,000 epochs. For training, we sampled 100 equally spaced points within the domain $[0,1]$ and used the two boundary points to apply Dirichlet conditions, giving $N_D=100$ and $N_{\partial D}=2$.  The results, shown in Figure \ref{fig:Poissonequation}, illustrate that Wav-KINNs effectively capture both high- and low-frequency components of the Poisson equation solution. Moreover, this approach achieves comparable accuracy with fewer parameters than the method based on MFF \cite{WANG2021113938fourierfeaturespINNs}.
\begin{figure}[ht]
    \centering
    \subfigure[]{
    \includegraphics[scale=0.45]{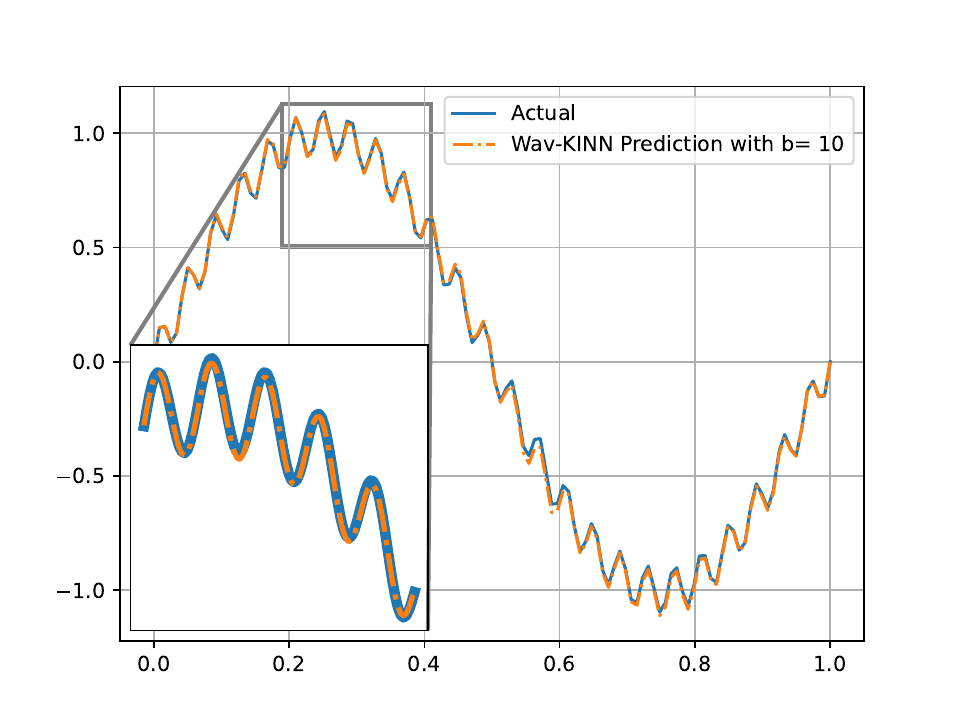}
    \label{fig:apporchingpoissonequationWav-KINN}
    }
    \subfigure[]{
    \includegraphics[scale=0.45]{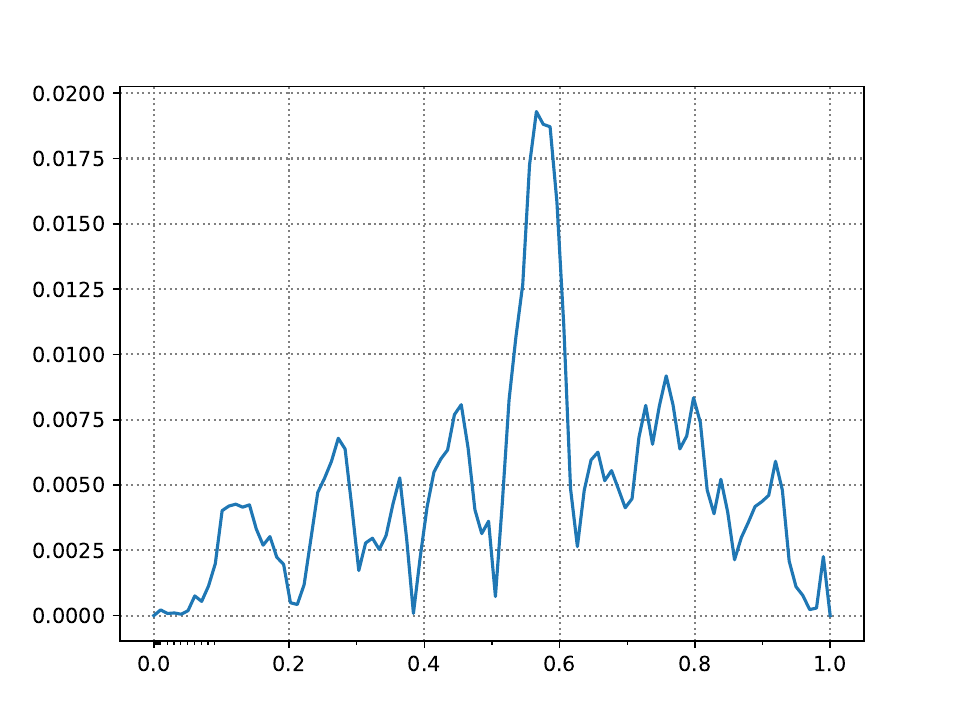}
    \label{fig:poissonequationerror}
    }
\caption{Approximation of the Poisson equation with high-frequency components using a Wav-KINN. a) Solution approximation by the Wav-KINN with two layers, shaped $[1,64,1]$, using the Morlet mother wavelet $\psi(x) = e^{-x^2} \cos(bx)$ with $b=10$. b) Absolute error between the Wav-KINN approximation and the actual solution.}
    \label{fig:Poissonequation}
\end{figure}

Thus, Wav-KINNs suggest they are less prone to spectral bias. To further test this, we now apply a more challenging equation: the time-dependent 1D Heat equation, given by:

\begin{equation}\label{eq:Heatequation}
    \begin{split}
u_t&=\frac{1}{(50\pi)^2} u_{xx}\text{ } x\in[0,1], \text{ }t\in [0,1],\\
u(0,t)&=u(1,t)=0,\\
u(x,0)&=\sin(50\pi x),
\end{split}
\end{equation}
whose solution is $u(x,t)=e^{-t}\sin(50\pi x)$.

To approximate the solution of this equation, we used a Wav-KINN with three layers in the configuration $[2,15,15,1]$, using the Morlet mother wavelet defined by $\psi(x) = e^{-x^2} \cos(bx)$ with $b=5$. The model was trained with the LBFGS optimizer for 10,000 epochs. Training data included $32 \times 32$ randomly selected points within the domain $[0,1]^2$, 200 boundary points for applying Dirichlet conditions, and 100 points for the initial conditions, resulting in $N_D = 1024$, $N_{bc} = 200$, and $N_{ic} = 100$. In this approach, boundary and initial conditions are treated separately, so the loss function $\mathcal{L}\theta$ contains three terms instead of two: $\mathcal{L}_D$ for the domain operator, $\mathcal{L}{bc}$ for boundary conditions, and $\mathcal{L}{ic}$ for initial conditions.

The results in Figure \ref{fig:Heatequation} illustrate that Wav-KINNs can successfully approximate both high- and low-frequency components in this more complex equation. In contrast, alternative methods to address spectral bias, such as those proposed in \cite{WANG2021113938fourierfeaturespINNs}, may require separating spatial and temporal domains and constructing STMFF to achieve similar performance with high spatial frequency components, thereby introducing additional hyperparameters. Wav-KINNs, however, achieve accurate solutions with fewer parameters and without partitioning domains, simplifying the model setup and parameter selection.

\begin{figure}[ht]
    \centering
    \subfigure[]{
    \includegraphics[scale=0.6]{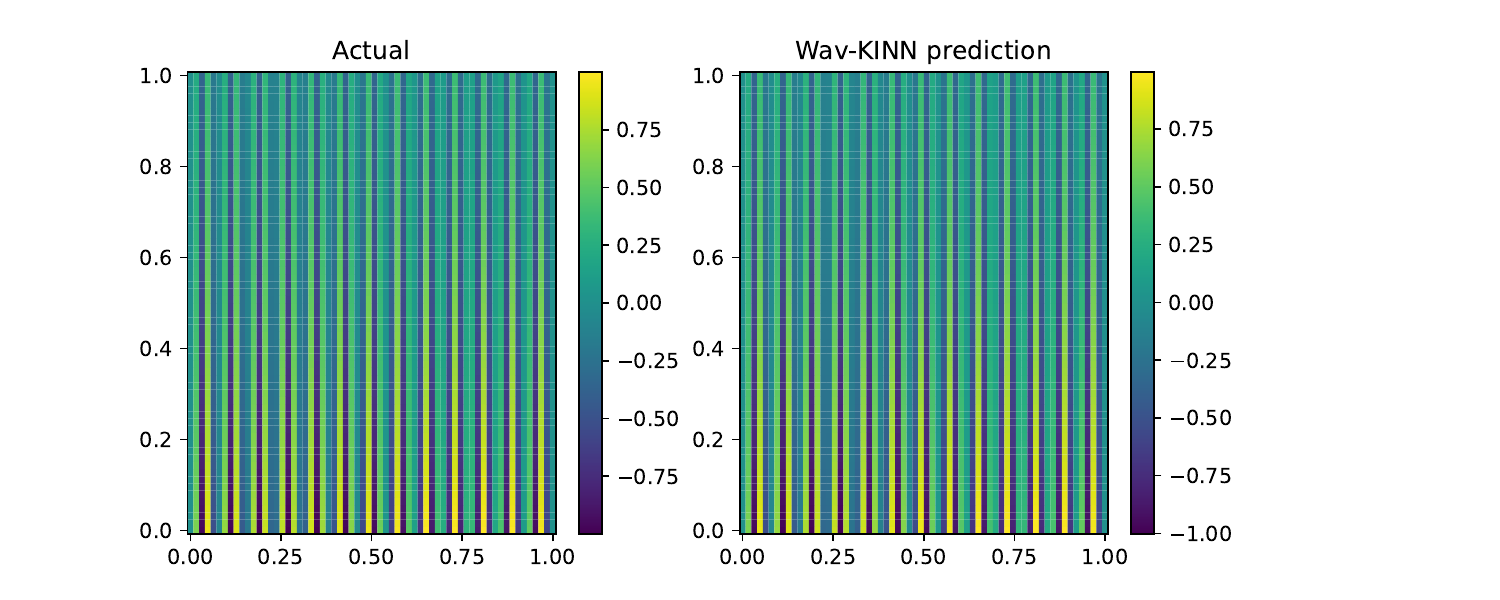}
    \label{fig:approachheatequation}
    }
    \subfigure[]{
    \includegraphics[scale=0.35]{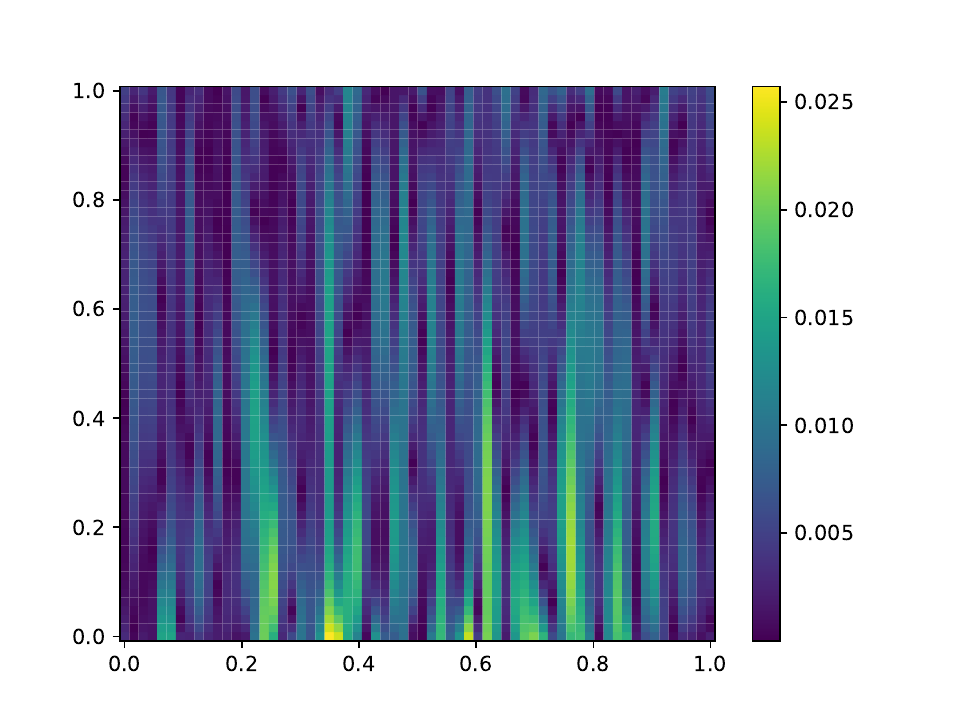}
    \label{fig:errorHeatequation}
    }
\caption{Approximation of the Heat equation with high-frequency components using a Wav-KINN. a) Solution approximation by the Wav-KINN with three layers, shaped $[1,15,15,1]$, using the Morlet mother wavelet $\psi(x) = e^{-x^2} \cos(bx)$ with $b=5$. b) Absolute error between the Wav-KINN approximation and the actual solution.}
    \label{fig:Heatequation}
\end{figure}

To further study how Wav-KINNs handle spectral bias, we test them with the Helmholtz equation:

\begin{equation}
\begin{split}
u_{xx}+u_{yy} + k^2u &= q(x, y),\\
u(-1,y)=u(1,y)&=u(x,-1)=0,u(x,1)=0.\\
\end{split}
\end{equation}

The exact solution is:

$u(x, y) =\sin(a_1\pi x) \sin(a_2\pi y),$ $$q(x,y)= -(a_1\pi )^2\sin(a_1\pi x) \sin(a_2\pi y)-(a_2\pi )^2\sin(a_1\pi x) sin(a_2\pi y) + k^2 \sin(a_1\pi x)\sin(a_2\pi y),$$ $a_1=1$ and $a_2=20$.

Using a Wav-KINN with four layers in the configuration $[2,15,15,15,1]$, we trained the model on $50 \times 50$ randomly selected points within the domain $[-1,1]^2$ and $200$ boundary points for Dirichlet conditions, resulting in $N_D = 2500$, $N_{\partial D} = 200$. The Morlet mother wavelet was defined by $\psi(x) = e^{-x^2} \cos(bx)$ with $b=5$, and the model was trained with the LBFGS optimizer for 10,000 epochs. 

The results are displayed in \ref{fig:Helmotzequation}. However, unlike the previous cases, the Wav-KINN struggled to accurately approximate the solution to the Helmholtz equation, even when adjusting the mother wavelet frequency, as seen in Figures \ref{fig:wrongHelmotzeqation} and \ref{fig:wrongHelmotzequationerror}. This performance issue may be due to an imbalance between the loss function terms 
$\mathcal{L}_D$ and $\mathcal{L}_{\partial D}$, where the Wav-KINN optimizes one term faster than the other, resulting in a suboptimal solution. This imbalance could be analyzed through the NTK of Wav-KINNs, following the methods in \cite{WANG2021113938fourierfeaturespINNs} and \cite{cao2020understandingspectralbiasdeep}. However, calculating the NTK for Wav-KINNs in this context is even more complex than for Wav-KANs, which is already computationally intensive. Therefore, we opted for an empirical approach, observing the training behavior of Wav-KINNs across epochs. 

To address the imbalance, we introduce weights $\lambda_D$ and $\lambda_{\partial D}$ in the loss function to balance the terms, such that:

\begin{equation}
    \mathcal{L}_\theta =\lambda_D \mathcal{L}_D + \lambda_{\partial D} \mathcal{L}_{\partial D}.
\end{equation}

By carefully choosing $\lambda_D$ and $\lambda_{\partial D}$ 
to balance the terms, we were able to obtain a more accurate approximation to the Helmholtz equation solution with Wav-KINNs, as shown in Figures \ref{fig:goodHelmotzequation} and \ref{fig:gooHelmotzequationerror}. This suggests that appropriate weighting in the loss function can help Wav-KINNs handle complex equations without excessive computational adjustments. 

\begin{figure}[ht]
    \centering
    \subfigure[]{
    \includegraphics[scale=0.6]{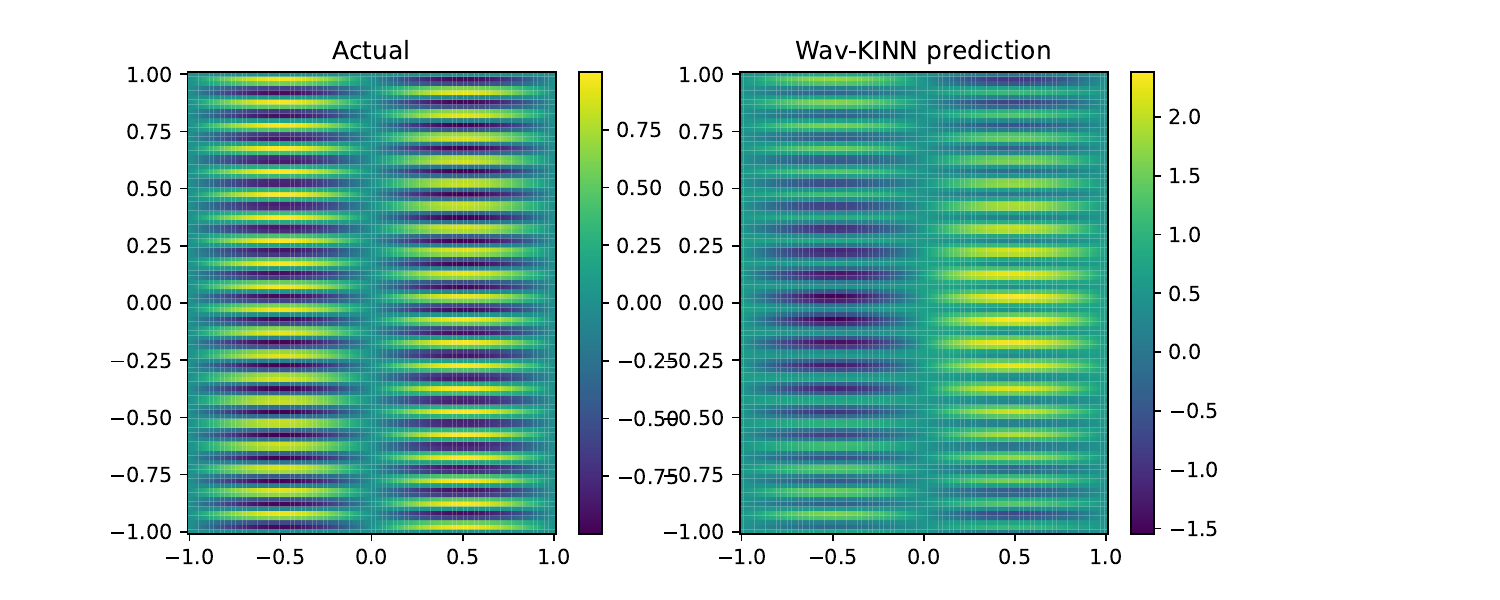}
    \label{fig:wrongHelmotzeqation}
    }
    \subfigure[]{
    \includegraphics[scale=0.35]{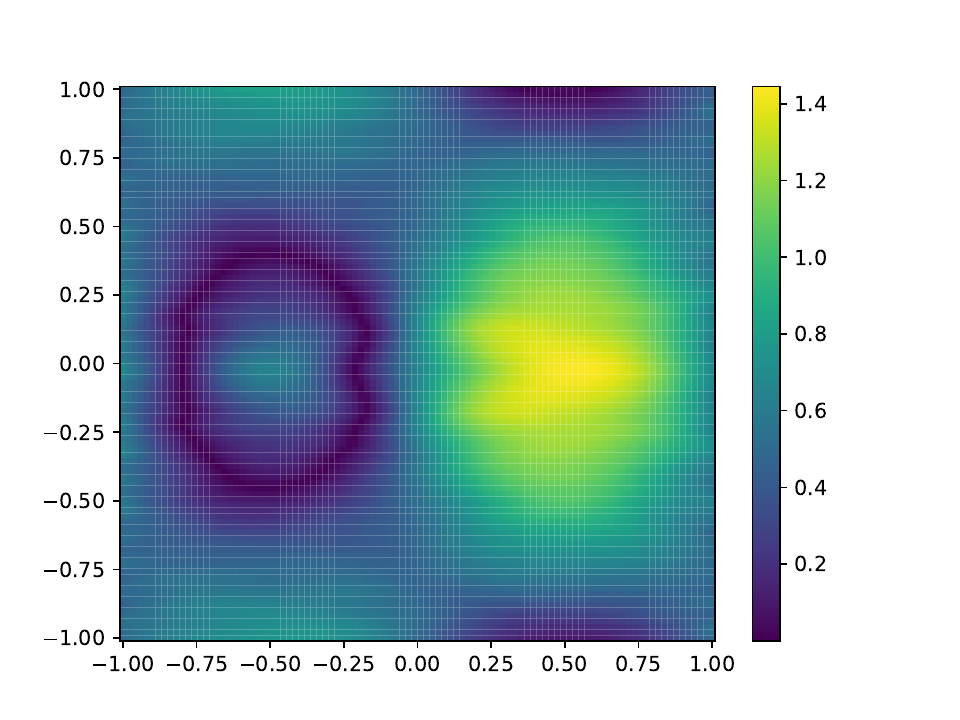}
    \label{fig:wrongHelmotzequationerror}
    }
    \subfigure[]{
    \includegraphics[scale=0.6]{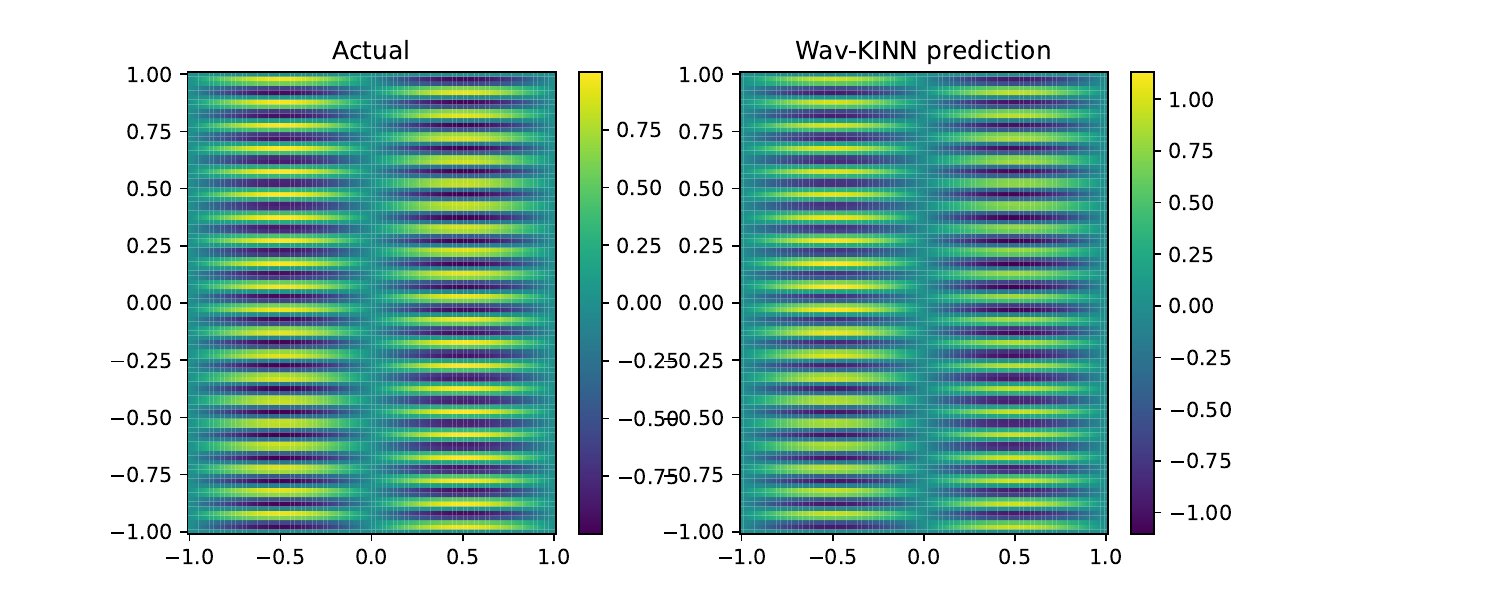}
    \label{fig:goodHelmotzequation}
    }
    \subfigure[]{
    \includegraphics[scale=0.35]{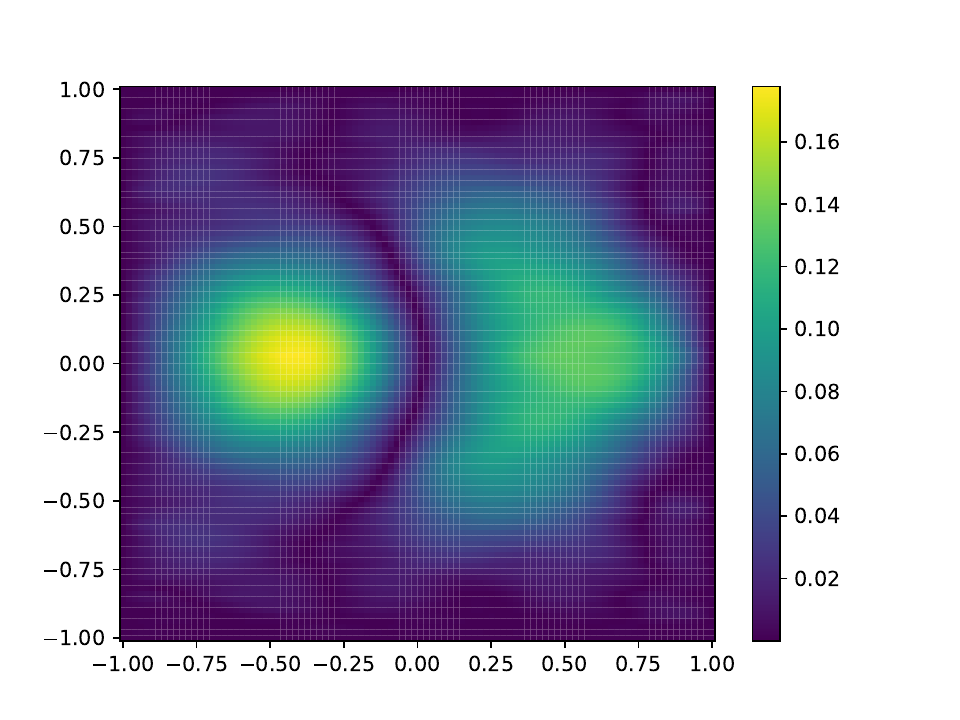}
    \label{fig:gooHelmotzequationerror}
    }
\caption{a) Solution approximation of the Helmholtz equation with high-frequency components using a Wav-KINN with four layers, structured as $[2,15,15,15,1]$, and the Morlet mother wavelet $\psi(x) = e^{-x^2} \cos(bx)$ with $b=5$, without balancing the loss terms $\mathcal{L}_D$ and $\mathcal{L}_{\partial D}$. b) Absolute error between the Wav-KINN approximation and the solution without balanced loss terms. c) Solution approximation of the Helmholtz equation using Wav-KINNs with the same parameters as in a), but with balanced loss terms, applying weights $\lambda_D$ and $\lambda_{\partial D}$. d) Absolute error between the Wav-KINN approximation and the solution with balanced loss terms.}
    \label{fig:Helmotzequation}
\end{figure}

To complete our empirical analysis, we tested the hyperbolic wave equation given by:

\begin{equation}
\begin{split}
u_{tt}&=4 u_{xx}\\
u(0,t)&=u(1,t)=0\\
u(x,0)&=\sin(πx) + 0.5\sin(4\pi x),\\
u_t(x,0)&=0.
\end{split}
\end{equation}
\noindent
where the exact solution is:

$u(x,t) = \sin(πx) \cos(2πt) + 0.5
\sin(4πx) \cos(8πt).$

This experiment used a Wav-KINN with four layers structured as $[2,20,20,20,1]$. The model was trained on $32 \times 32$ randomly selected points within the domain $[0,1]^2$, with 200 boundary points for Dirichlet conditions, 200 points for the initial condition, and 200 points for the Neumann condition, resulting in $N_D = 2500$, $N_{bc} = 200$, $N_{ic} = 200$, and $N_{nbc} = 200$. Each condition was handled separately, leading to a loss function with four terms: $\mathcal{L}_D$, $\mathcal{L}_{bc}$, $\mathcal{L}_{ic}$, and $\mathcal{L}_{nbc}$. The Morlet mother wavelet was defined by $\psi(x) = e^{-x^2} \cos(bx)$ with $b=5$, and training was conducted using the LBFGS optimizer for 10,000 epochs. 

Similar to previous results, Wav-KINNs initially struggled to approximate the solution to the wave equation, as shown in Figures \ref{fig:waveequationwrong} and \ref{fig:waveequationwrongerror}. However, by balancing the terms in the loss function with weights $\lambda_{D}$, $\lambda_{bc}$, $\lambda_{ic}$, and $\lambda_{nbc}$, Wav-KINNs were able to approximate the solution more accurately, as demonstrated in Figures \ref{fig:waveeqautiongood} and \ref{fig:waveeqatuinogooderror}.

\begin{figure}[ht]
    \centering
    \subfigure[]{
    \includegraphics[scale=0.6]{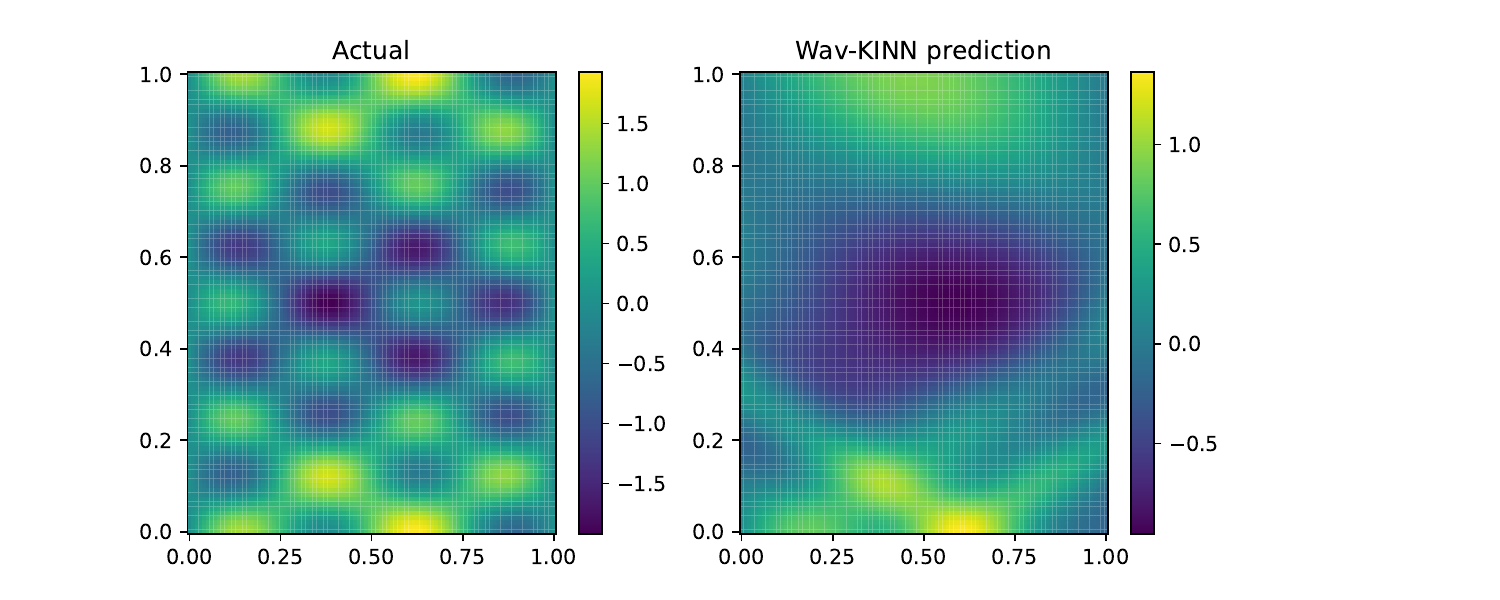}
    \label{fig:waveequationwrong}
    }
    \subfigure[]{
    \includegraphics[scale=0.35]{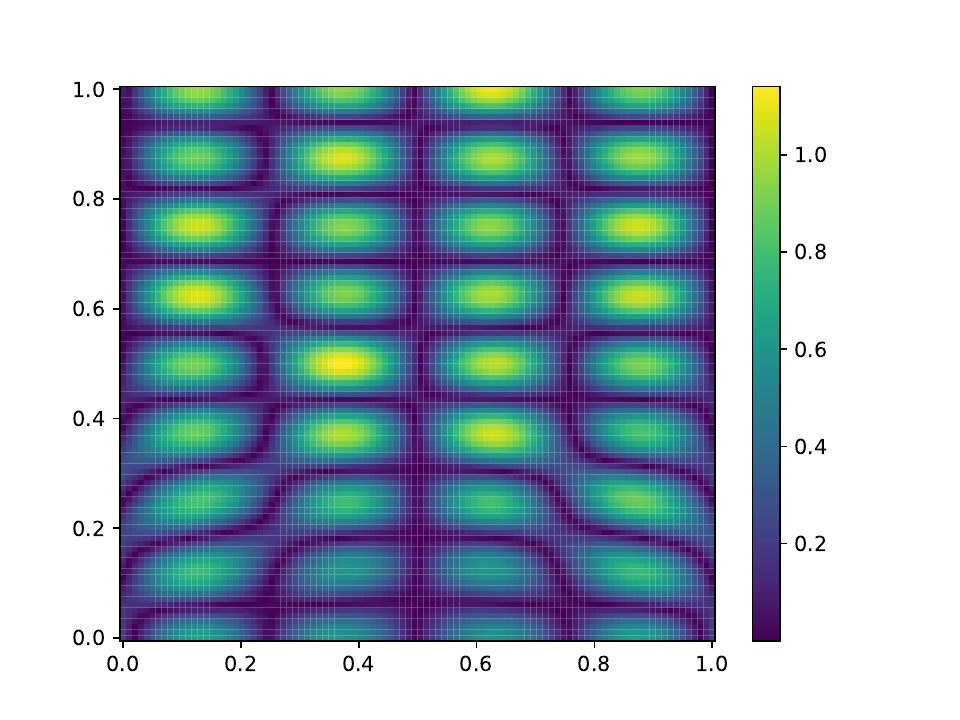}
    \label{fig:waveequationwrongerror}
    }
    \subfigure[]{
    \includegraphics[scale=0.6]{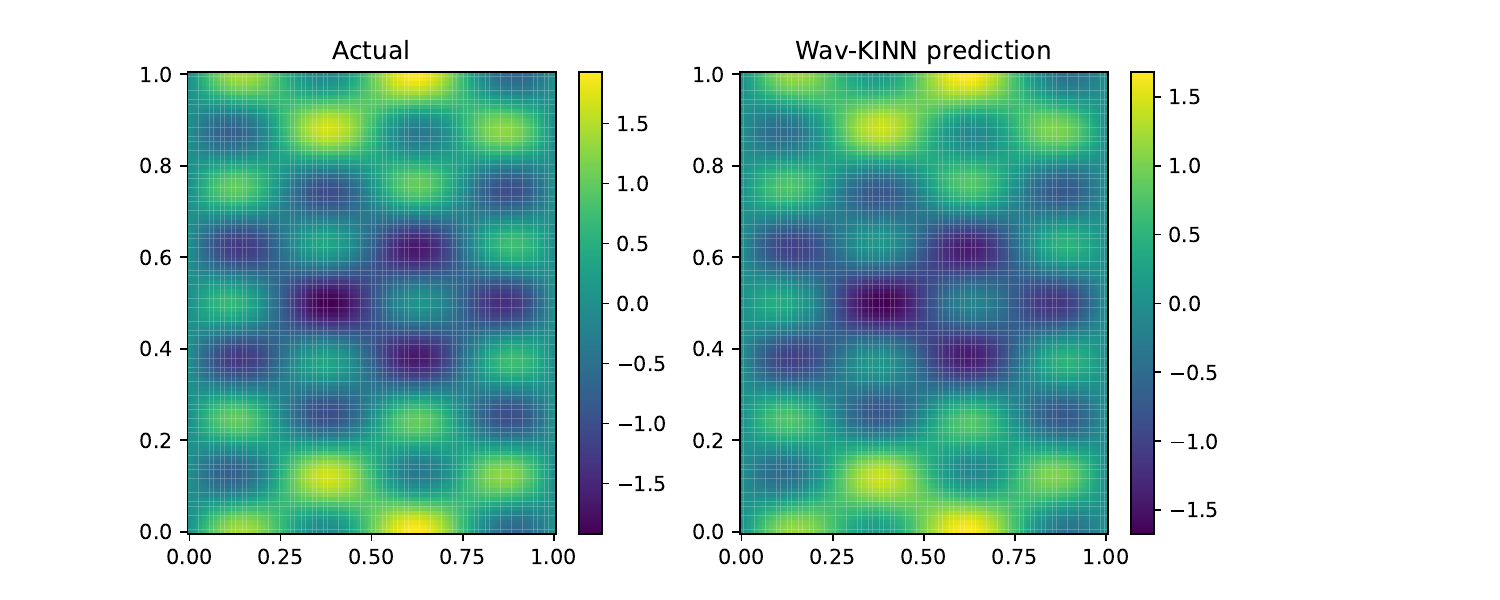}
    \label{fig:waveeqautiongood}
    }
    \subfigure[]{
    \includegraphics[scale=0.35]{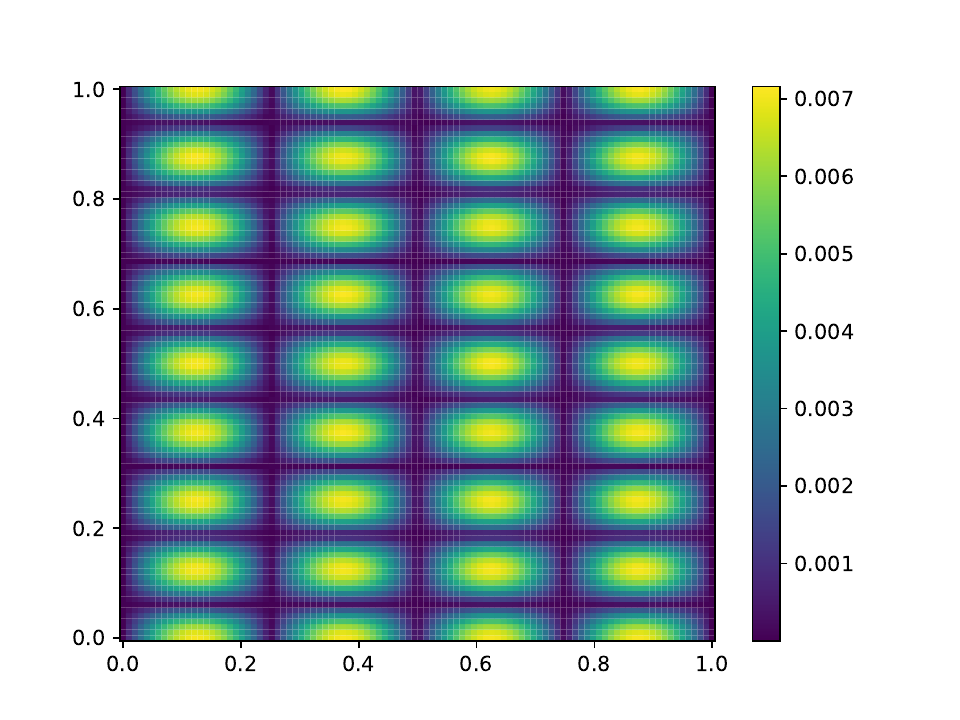}
    \label{fig:waveeqatuinogooderror}
    }
\caption{a) Solution approximation of the Wave equation with high-frequency components using a Wav-KINN with four layers, structured as $[2,20,20,20,1]$, and the Morlet mother wavelet $\psi(x) = e^{-x^2} \cos(bx)$ with $b=5$, without balancing the loss terms. b) Absolute error between the Wav-KINN approximation and the solution without balanced loss terms. c) Solution approximation of the Wave equation using Wav-KINNs with the same parameters as in a), but with balanced loss terms, applying weights $\lambda_{D}$, $\lambda_{bc}$, $\lambda_{ic}$, and $\lambda_{nbc}$. d) Absolute error between the Wav-KINN approximation and the solution with balanced loss terms.}
    \label{fig:waveequation}
\end{figure}

To conclude, by testing Wav-KINNs on the three main types of equations with high-frequency components, we demonstrate that Wav-KINNs effectively mitigate spectral bias without the need to separate domains or introduce additional artifacts.

\section{Discussion}

We have analyzed the behavior of the NTK eigenvalues for Wav-KANs and demonstrated that they can be effectively controlled by adjusting the frequency of the mother wavelet, thereby avoiding spectral bias. We have shown empirically that Wav-KANs do not exhibit spectral bias even with an unequal increase in hidden units. Although fewer parameters are required for Wav-KANs to approximate functions compared to traditional deep NNs, we observed that setting fixed translation and scaling parameters with $T\in [0,1]$ and $S=1$ allows Wav-KANs to perform well while minimizing the parameter count. However, a more detailed study is essential to understand the extent to which these parameters can be fixed or varied, as well as the boundaries for extending the parameters and assumptions from Proposition \ref{prop3}. This exploration is of significant interest, as it could provide insights into optimal parameter selection in Wav-KANs, potentially reducing computational costs and enhancing model efficiency.
On the other hand, implementing Wav-KANs within PINNs, i.e., forming Wav-KINNs, has demonstrated strong performance in approximating solutions for differential equations with high-frequency components, where classical PINNs typically struggle. However, examining spectral bias in Wav-KINNs theoretically through the lens of NTK analysis is complex, necessitating further in-depth study. A deeper investigation of spectral bias in Wav-KINNs could provide valuable insights into their ability to approximate functions and their derivatives without bias, reinforcing their application to a broader range of high-frequency differential equation problems. This understanding is key to optimizing Wav-KINNs for complex systems where high-frequency accuracy is critical. 


As shown in the Helmholtz and Wave equation approximations, while Wav-KINNs can handle high-frequency components, they may encounter challenges with unbalanced terms in the loss function. For computational efficiency, manually adjusting weights can help balance these terms; however, further research is needed to determine optimal strategies for selecting these hyperparameters. Additionally, exploring adaptive weight adjustments at each learning step could enhance accuracy while preserving efficiency and avoiding increases in computational cost. Developing such strategies would be valuable for improving Wav-KINNs in applications that demand balanced accuracy across all terms.

Furthermore, while Wav-KINNs have demonstrated robust performance in approximating functions within rectangular and simple domains, their application to more complex geometries remains underexplored. Similar issues to those encountered by Spl-KINNs \cite{wang2024kolmogorovarnoldinformedneural}, such as difficulties in modeling intricate boundaries and shapes, may arise. Accurately capturing complex geometries and boundaries is crucial for simulations in advanced scientific and engineering contexts, emphasizing the need for future studies to address these limitations.

In light of these considerations, it is interesting to examine the adaptability of Wav-KINNs to high-dimensional implementations. High-dimensional problems frequently involve complexities associated with the curse of dimensionality, which can influence the training and convergence of neural networks. Investigating how Wav-KINNs might address these challenges—potentially through techniques such as dimensionality reduction or adaptive sampling strategies—could provide valuable insights. A more thorough investigation into these topics could facilitate a more comprehensive understanding of the potential and limitations of Wav-KINNs and their applicability to a broader range of problems in scientific computing and engineering disciplines.

\section{Conclusion}
In this study, we investigated the potential of Wav-KANs in approximating functions and physics-informed settings. We mainly focused on their capacity to mitigate spectral bias, a prevalent limitation in traditional neural networks when handling high-frequency components in approximation functions and solutions to differential equations. By leveraging the controlled frequency properties of mother wavelets and implementing Wav-KANs in PINNs, a variant called Wav-KINNs was developed and applied to a series of partial differential equations (PDEs), including the Poisson, Heat, Wave, and Helmholtz equations.

The results indicate that Wav-KINNs exhibit robust performance in approximating solutions across various PDE types. In contrast to standard PINNs, which frequently encounter difficulties in handling high-frequency components and necessitate the introduction of additional hyperparameters or domain decomposition (as in the STMFF approach), Wav-KINNs are capable of achieving accurate solutions with a reduced number of parameters and without the necessity of separating spatial and temporal domains. This reduction in complexity represents a notable advantage, demonstrating that Wav-KINNs can effectively address spectral bias through an elegant approach that retains computational efficiency.

However, as suggested by the results of the Helmholtz and Wave equation tests, Wav-KINNs, like other PINN architectures, may encounter difficulties when terms in the loss function become imbalanced, particularly between domain and boundary conditions. This issue is a challenge in the broader PINN framework. Addressing this imbalance is essential for future research, as it could help stabilize the training process and improve overall performance. While analyzing this imbalance through the NTK could provide valuable insights, the analytical computation of the NTK in wavelet-based architectures like Wav-KINNs remains computationally prohibitive, highlighting another avenue for further exploration. 

In summary, Wav-KANs and the  Wav-KINNs present a promising model for addressing spectral bias in function approximation and PINNs, enabling a stable and adaptable approach to solving high-frequency, intricate differential equations. Further research could concentrate on developing theoretical analyses of the NTK for wavelet-based networks, refining loss-balancing strategies, and extending the applications of Wav-KINNs in real-world physics-informed problems.
\newpage

\nocite{*}
\bibliography{references}

\newpage
\appendix

\section{Proof Proposition 1} \label{Appendix A}
\begin{proof}
    Since $K(x^r,x^s)=\sum_{i=1}^n \psi_i(x_i^r)\psi_i(x_i^s)$ by equation \eqref{eq:kerneldotproduct},   replacing this expression into the integral equation \eqref{eq:integralequation} gives:
    \begin{equation}
        \begin{split}
            \int_C \sum_{i=1}^n \psi_i(x_i^r)\psi_i(x_i^s) g(x^s) dx^s &= \lambda g(x^r)\\
            \sum_{i=1}^n \int_C  \psi_i(x_i^r)\psi_i(x_i^s) g(x^s) dx^s &= \lambda g(x^r)
        \end{split}
    \end{equation}
Taking derivatives with respect to $x_k^r$ on both sides of the equation, we get:
\begin{equation}
    \begin{split}
        \sum_{i=1}^n \frac{\partial   \psi_i}{\partial x_k^r}(x_i^r) \int_C\psi_i(x_i^s) g(x^s) dx^s &= \lambda \frac{\partial g }{\partial x_k^r }(x^r)\\
        \frac{1}{S_k}\frac{d   \psi_k}{d x_k^r}(x_k^r) \int_C\psi_k(x_k^s) g(x^s) dx^s &= \lambda \frac{\partial g }{\partial x_k^r }(x^r)\\
    \end{split}
\end{equation}
Now, if there exists a function $\psi_k$ such that $\frac{d\psi_k}{dx_k}(x_k)=\omega_k(x_k)\psi_k(x_k)$,  we first note that $\omega_{k}(x_k)$ is a univariate function that cannot be constant. This is because if $\omega_{k}$ were a constant, then for a linear ordinary differential equation:
\begin{equation}
\begin{split}
    \psi_k(x_k)=C_1e^{\omega_k x_k},
\end{split}
\end{equation}
but no constant $\omega_k$  exists such that $\psi_k$ which is a translation and scaling of the mother wavelet $\psi$ retains the properties of the mother wavelet. Thus, $\omega_k(x_k)$ cannot be constant.

Substituting the expression for \( \frac{d \psi_k}{dx_k} \), we get:
\begin{equation}
    \frac{\omega_k(x_k^r)}{S_k} \int_C \psi_k(x_k^r) \psi_k(x_k^s) g(x^s) \, dx^s = \lambda \frac{\partial g}{\partial x_k^r}(x^r).
\end{equation}
Now, let 
\[    h_k(x_k^r) = \begin{cases}        \frac{S_k}{\omega_k(x_k^r)} & \text{if } \omega_k(x_k^r) \neq 0, \\        0 & \text{otherwise}.    \end{cases}\]
Then, we can rewrite the previous equation as:
\begin{equation}
    \int_C \psi_k(x_k^r) \psi_k(x_k^s) g(x^s) \, dx^s = \lambda h_k(x_k^r) \frac{\partial g}{\partial x_k^r}(x^r).
\end{equation}
Summing over all partial derivatives gives:
\begin{equation}
    \sum_{i=1}^n \int_C \psi_i(x_i^r) \psi_i(x_i^s) g(x^s) \, dx^s = \lambda \sum_{i=1}^n h_i(x_i^r) \frac{\partial g}{\partial x_i^r}(x^r).
\end{equation}
Hence, we have:
\begin{equation}
    \lambda g(x^r) = \lambda \langle \omega_0(x^r), \nabla g(x^r) \rangle,
\end{equation}
where \( \omega_0(x^r) = (h_1(x_1^r), h_2(x_2^r), \dots, h_n(x_n^r))^T \).
\end{proof}
\section{Proof Proposition 2}\label{Appendix B}
\begin{proof}
    Suppose that $g(x^r)$ is the corresponding eigenfunction of the operator $K $ with \( K(x^r, x^s) = \psi_1(x^r) \psi_1(x^s) \). Then, by Proposition \ref{pro1}, we have:
    \begin{equation}
        g(x^r) = h_1(x^r) \frac{d g}{dx^r},
    \end{equation}
    where \( h_1(x^r) \) is as described in Proposition \ref{pro1}. Integrating this equation, we get:
    \begin{equation}
        g(x^r) = C_2 e^{\frac{1}{S} \int \omega_1(x^r) \, dx^r},
    \end{equation}
    where \( \omega_1(x^r) \) is a function that satisfies:
    \begin{equation}
        \frac{d \psi_1}{dx^r}(x^r) = \omega_1(x^r) \psi_1(x^r).
    \end{equation}
    Therefore, solving the differential equation for \( \psi_1(x) \) gives:
    \begin{equation}
        \psi_1(x^r) = C_1 e^{\int \omega_1(x^r) \, dx^r}.
    \end{equation}
    Substituting this result back into the expression for \( g(x^r) \), we get:
    \begin{equation}
        g(x^r) = C_3 \psi_1^{\frac{1}{S}}(x^r),
    \end{equation}
    where \( C_3 \) is a constant.
\end{proof}

\section{Proof Proposition 3}\label{Appendix C}
\begin{proof}
    By Proposition \ref{prop2}, substituting \( g(x^r) \) into the integral equation \eqref{eq:integralequation}, we obtain:
    \begin{equation}
        \lambda = \int_C \psi_1^{S+1}(x^r) \psi_1^{\frac{S+1}{S}}(x^s) \, dx^s,
    \end{equation}
    for a compact domain \( C \).

    On the other hand, if we choose:
    \begin{equation}
        \omega_1(x^r) = \frac{-bS \tan\left(b \frac{x^r-T}{S}\right) - (x^r-T)}{S^2},
    \end{equation}
    and since \( \frac{d \psi_1}{dx^r}(x^r) = \omega_1(x^r) \psi_1(x^r) \), solving the differential equation and substituting \( \omega_1 \), we get:
    \begin{equation}
        \psi_1(x^r) = e^{-\frac{(x^r-T)^2}{2S^2}} \cos\left(b \frac{x^r-T}{S}\right),
    \end{equation}
    which is a translation and scaling of \( \psi(x) = e^{-x^2/2} \cos(bx) \), commonly known as the Morlet Wavelet. Thus, choosing the compact domain \( C = [0,1] \) and substituting into equation (36):
    \begin{equation}
        \lambda = e^{-\frac{(x^r-T)^2 (S+1)}{2S^2}} \cos^{S+1}\left(b \frac{x^r-T}{S}\right) \int_0^1 e^{-\frac{(x^s - T)^2 (S+1)}{2S^3}} \cos^{\frac{S+1}{S}}\left(b \frac{x^s - T}{S}\right) \, dx^s.
    \end{equation}
    
    Assuming that \( x \in [0,1] \), we can consider a small translation \( T \) within the same interval, i.e., \( T \in [0,1] \). Let us assume that the weights can adapt the mother function \( \psi \) and set \( S = 1 \). Therefore, since \( \cos^2(b(x - T)) \) provides a controlling factor, \( \lambda \) is bounded from below by:
    \begin{equation}
        \frac{1}{4} e^{-4b^2 (x^r-T)^2} \leq \lambda.
    \end{equation}
\end{proof}
\section{Study of different mother wavelets }\label{Appendix D}
In our study, we primarily used the Morlet as the mother wavelet, following the assumptions outlined in Proposition \ref{pro1}. However, we also briefly explored other wavelet functions, presented below, to verify that controlling the frequency in the mother wavelet allows us to control the decay rate of the NTK eigenvalues in Wav-KANs. These include wavelets without an explicit frequency parameter, such as the Mexican Hat Wavelet and the Derivative of the Gaussian Wavelet, which are commonly used alternatives.

\subsection*{Shannon Wavelet}

We begin with the real Shannon wavelet, explicitly defined in the frequency domain. The Shannon wavelet is designed to be localized in the frequency space, making it practical for analyzing specific frequency bands. In the time domain, the real Shannon wavelet is related to the $sinc$ function (the inverse Fourier transform of a rectangular window in the frequency domain):

\begin{equation} \psi(x) = \frac{\sin(\omega_1 x) - \sin(\omega_2 x)}{\pi x}. \end{equation}

This function exhibits oscillations and decay, capturing a frequency range determined by $\omega_1$ and $\omega_2$. To simplify control over the frequency parameter, we specifically use octave bands, where:

\begin{equation} \omega_1 = 2 \omega_2. \end{equation}

We then analyze the behavior of the NTK eigenvalues for different values of $\omega_1$ when approximating the function in equation \eqref{eq:solpoissoneq}, using the real Shannon wavelet as the mother wavelet. As shown in Figure \ref{fig:1D}, similar to the results obtained with the Morlet wavelet, increasing the parameter $\omega_1$ leads to a slower decay of the eigenvalues. This confirms that adjusting the frequency parameter allows us to effectively influence the behavior of the NTK eigenvalues.

\begin{figure}
    \centering
    \subfigure[]{
    \includegraphics[scale=0.45]{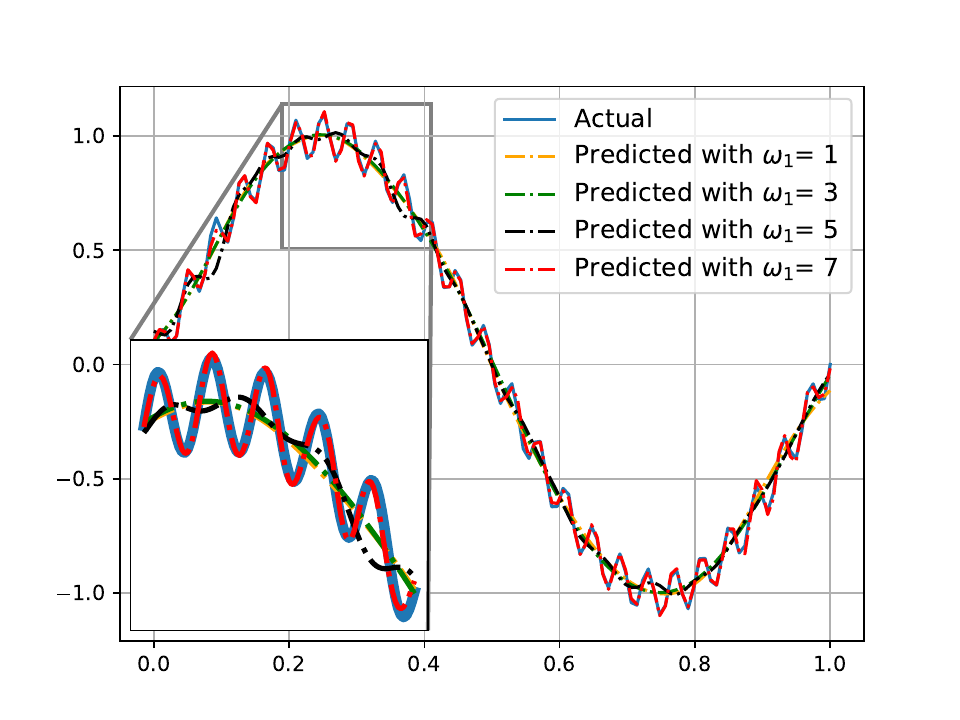}
    \label{fig:1Da}
    }
    \subfigure[]{
    \includegraphics[scale=0.45]{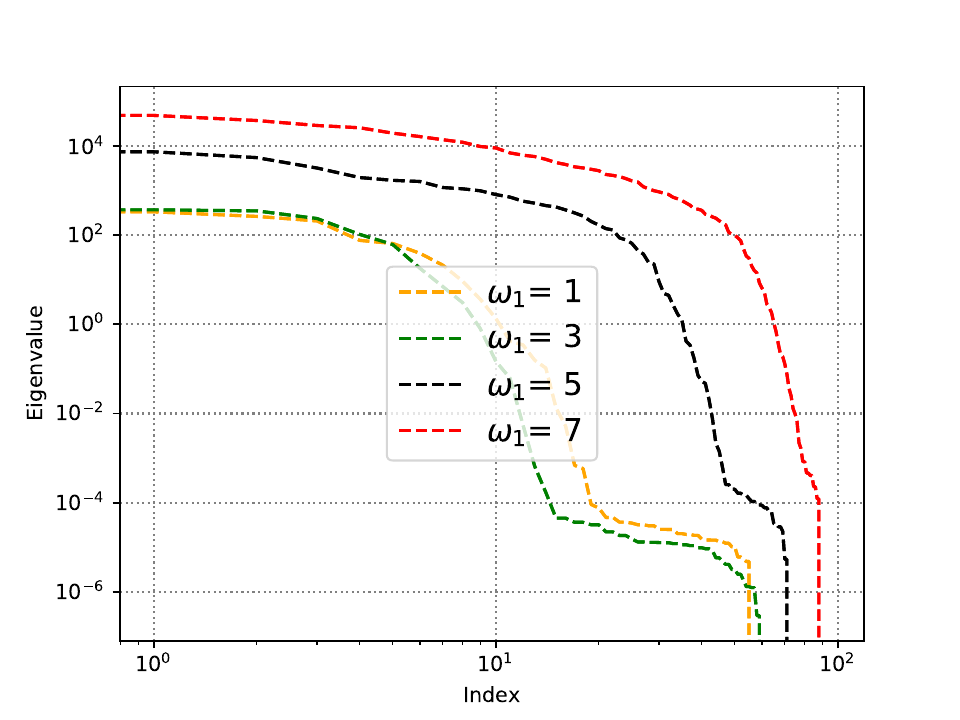}
    \label{fig:1Db}
    }
\caption{ a) Approximation of function \eqref{eq:solpoissoneq} using a Wav-KAN with two layers [1,35,1] and the real Shannon Wavelet as mother wavelet for different $\omega_1$ values: $1, 3, 5$ y $7$, trained for 1000 epochs. b) NTK eigenvalues in descending order for each value of $\omega_1$.}
    \label{fig:1D}
\end{figure}

\subsection*{Mexican hat wavelet}

On the other hand, the Mexican Hat wavelet given by:

\begin{equation}
    {\displaystyle \psi (x)={\frac {2}{{\sqrt {3\sigma }}\pi ^{1/4}}}\left(1-\left({\frac {x}{\sigma }}\right)^{2}\right)e^{-{\frac {x^{2}}{2\sigma ^{2}}}}},
\end{equation}
does not have an explicit frequency parameter, unlike the Morlet and Shannon wavelets. However, the Mexican Hat wavelet is widely used in wave analysis, and in many practical applications, the scale parameter $\sigma$ is adjusted to capture different frequency ranges. A larger value of $\sigma$ results in a broader wavelet, which captures lower frequencies, while a smaller $\sigma$ produces a narrower wavelet, targeting higher frequencies.

As shown in Figure \ref{fig:2D}, by varying the values of $\sigma$, we can control the rate at which the eigenvalues decrease to approximate the solution of equation \eqref{eq:solpoissoneq}. 

\begin{figure}
    \centering
    \subfigure[]{
    \includegraphics[scale=0.45]{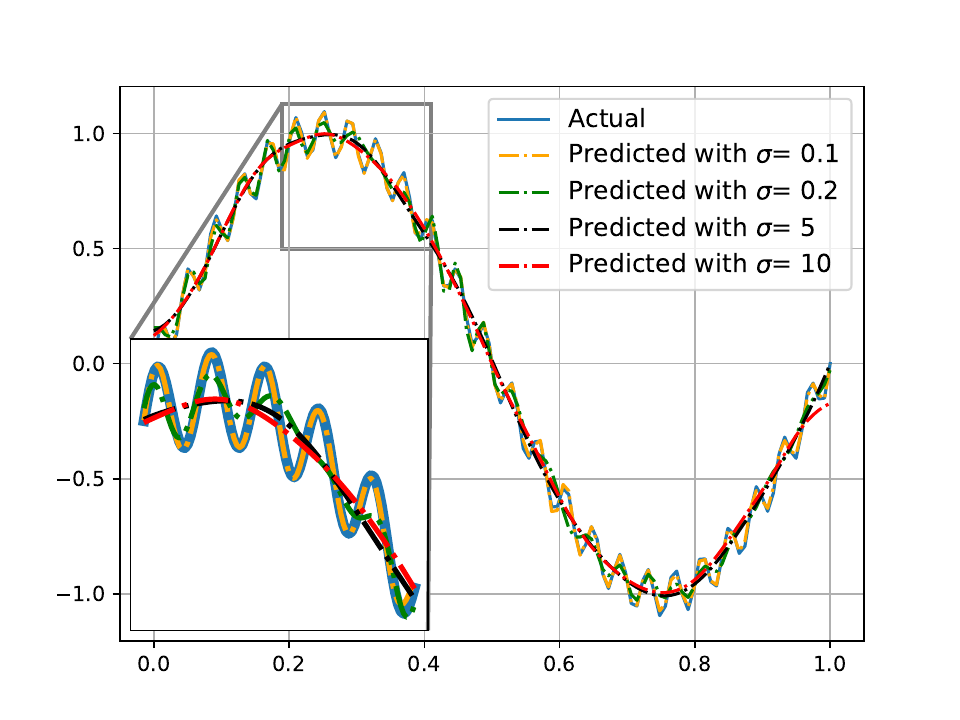}
    \label{fig:a2D}
    }
    \subfigure[]{
    \includegraphics[scale=0.45]{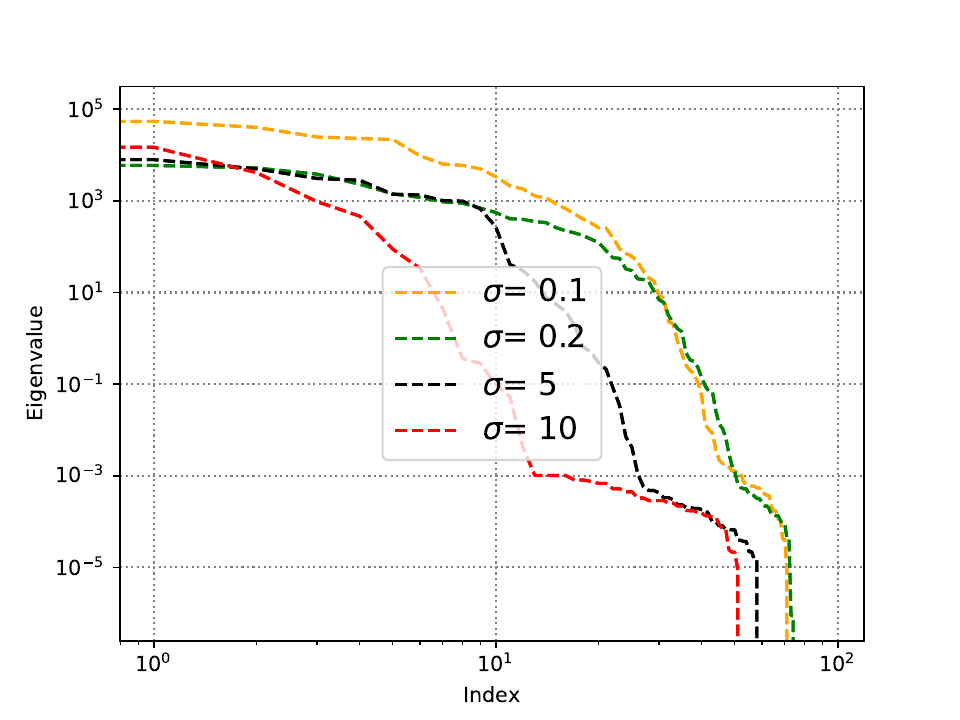}
    \label{fig:b2D}
    }
\caption{a) Approximation of function \eqref{eq:solpoissoneq} using a Wav-KAN with two layers [1,35,1] and the Mexican Hat Wavelet as the mother wavelet for different $\sigma$ values: $0.1, 0.2, 5$ y $10$, trained for 1000 epochs. b) NTK eigenvalues in descending order for each value of $\sigma$.}
    \label{fig:2D}
\end{figure}

\subsection*{Derivative of Gaussian wavelet}
Similar to the Mexican Hat wavelet, the Derivative of the Gaussian (DoG) wavelet given by:
\begin{equation}
    -\frac{x}{\sigma^2}e^{-{\frac {x^{2}}{2\sigma ^{2}}}},
\end{equation}
does not have an explicit frequency parameter. However, a larger $\sigma$ results in a broader Gaussian, causing the derivative to have a slower transition, which captures lower frequency content. Conversely, a smaller $\sigma$ produces a narrower Gaussian, leading to sharper peaks in the derivative and capturing higher frequency content.

As shown in Figure \ref{fig:3D}, by adjusting the values of $\sigma$, we can control the rate at which the eigenvalues of the NTK decrease.

\begin{figure}
    \centering
    \subfigure[]{
    \includegraphics[scale=0.45]{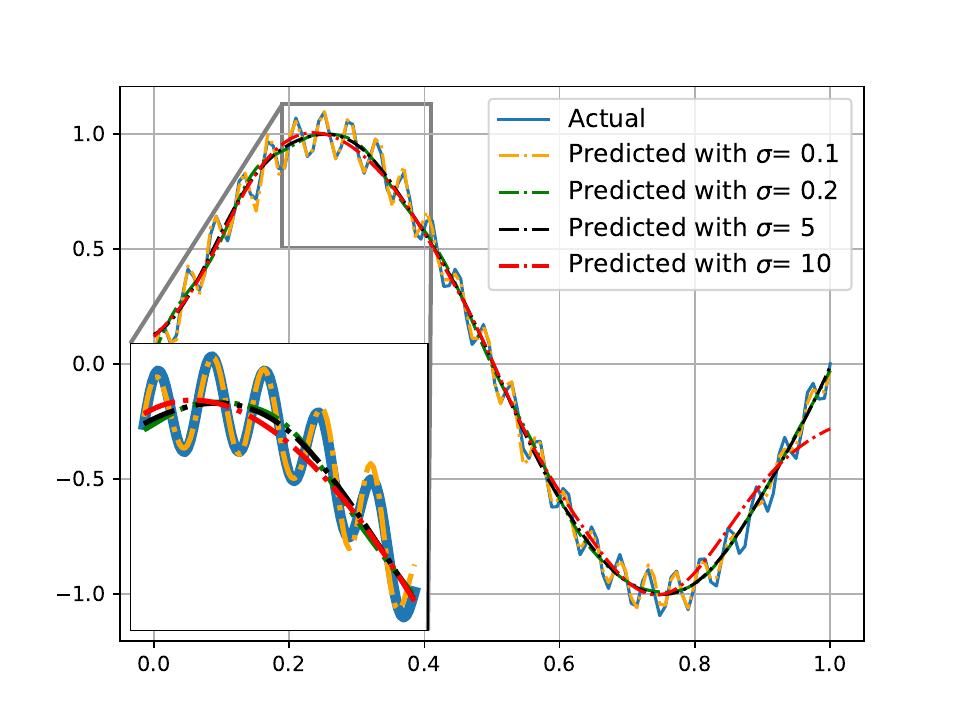}
    \label{fig:a3D}
    }
    \subfigure[]{
    \includegraphics[scale=0.45]{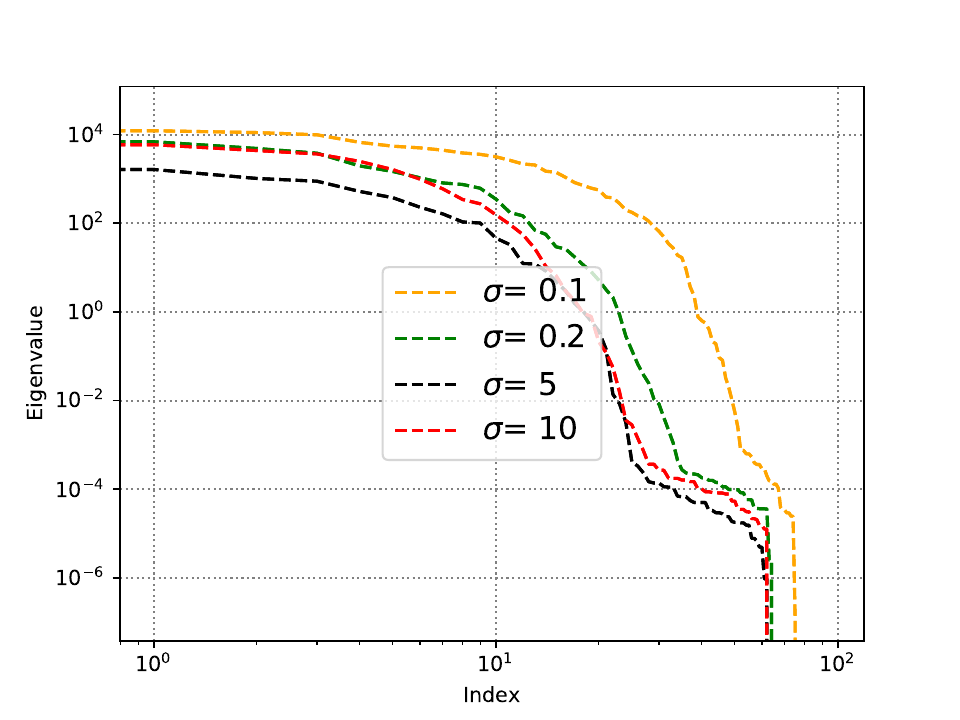}
    \label{fig:b3D}
    }
\caption{a) Approximation of function \eqref{eq:solpoissoneq} using a Wav-KAN with two layers [1,35,1] and the Derivative of Gaussian Wavelet as the mother wavelet for different $\sigma$ values: $0.1, 0.2, 5$ y $10$, trained for 1000 epochs. b) NTK eigenvalues in descending order for each value of $\sigma$.}
    \label{fig:3D}
\end{figure}


\end{document}